\documentclass{article}

\usepackage[final,main]{neurips_2026}

\usepackage[utf8]{inputenc}
\usepackage[T1]{fontenc}
\usepackage{hyperref}
\usepackage{url}
\usepackage{booktabs}
\usepackage{amsmath}
\usepackage{amsfonts}
\usepackage{amssymb}
\usepackage{nicefrac}
\usepackage{microtype}
\usepackage{xcolor}
\usepackage{graphicx}
\usepackage{wrapfig}
\usepackage{subcaption}
\usepackage[ruled,lined]{algorithm2e}
\usepackage{multirow}
\usepackage{enumitem}
\usepackage{pgfplots}
\usepackage{etoc}
\pgfplotsset{compat=1.18}
\usepackage[breakable]{tcolorbox}
\usepackage{amsthm}

\newtheorem{proposition}{Proposition}
\usepackage{listings}
\lstdefinestyle{promptstyle}{
  basicstyle=\small\ttfamily,
  breaklines=true,
  columns=fullflexible,
  keepspaces=true,
}

\title{
SquidAgent: Parallelize Wisely, Coordinate Efficiently}

\author{%
  Yexiong Lin\textsuperscript{1}\quad
  Shanshan Ye\textsuperscript{2}\thanks{Corresponding author: \texttt{shanshan.ye@mbzuai.ac.ae}.}\quad
  Yu Yao\textsuperscript{1}\\
  \textbf{Zhen Fang\textsuperscript{3}\quad
  Bo Han\textsuperscript{4}\quad
  Tongliang Liu\textsuperscript{1,2}}\\[0.5em]
  \normalfont\textsuperscript{1}The University of Sydney\\
  \normalfont\textsuperscript{2}Mohamed bin Zayed University of Artificial Intelligence\\
  \normalfont\textsuperscript{3}University of Technology Sydney\\
  \normalfont\textsuperscript{4}Hong Kong Baptist University
}

\begin{document}

\maketitle

\begin{abstract}
LLM-based agents solve complex multi-step tasks, but sequential execution incurs substantial latency. In principle, parallelizing work across multiple agents should yield near-linear speedups. Yet existing parallel multi-agent systems often run slower than a single-agent baseline. We attribute this gap to two hidden costs that parallel execution incurs but a serial agent avoids. First, there is a \emph{re-exploration cost}: redundant effort spent by parallel workers reconstructing context that the orchestrator already possesses, such as prior decisions, that would otherwise be inherited implicitly in a serial execution. Second, there is an \emph{alignment cost}: the overhead required to reconcile inconsistencies across independently generated outputs. We thus derive a principled decision criterion: a layer should be parallelized only when its critical-path cost, plus re-exploration and alignment overheads, is lower than the corresponding serial cost. While this criterion is naturally expressed in wall-clock time, we observe that LLMs are poorly calibrated when asked to estimate task duration. To address this, we instead measure cost in predicted output tokens, which we empirically find LLMs can estimate substantially more reliably than wall-clock time. Building on this token-based criterion, we propose SquidAgent. It estimates all token budgets in a single planning step, forks each worker directly from the orchestrator's session to eliminate re-exploration cost, and replaces post-hoc reconciliation with a pre-generated shared convention block that converts alignment into a bounded upfront cost. A deterministic scheduler then applies the criterion layer by layer. Empirically, SquidAgent achieves a 2.2$\times$ mean throughput improvement and a 2.6$\times$ mean wall-time speedup over Claude Code, and a 2.0$\times$ throughput improvement over the strongest multi-agent baseline.
\end{abstract}

\etocdepthtag.toc{mtchapter}
\section{Introduction}
\label{sec:intro}

LLM-based agent systems are increasingly deployed for complex, multi-step tasks such as software development~\citep{hong2024metagpt,qian2024chatdev}, document authoring, and structured planning. As tasks grow in scale and complexity, a single agent often becomes a bottleneck, motivating the rise of \emph{multi-agent systems} that coordinate multiple specialized agents to solve a task collaboratively. 
A common paradigm in these systems is to decompose a high-level user request into a set of subtasks organized as a directed acyclic graph (DAG)~\citep{zhuge2024gptswarm,qian2025scaling}, where edges encode dependencies between subtasks, and assign each subtask to a dedicated agent. 
This graph-based decomposition has been adopted by a wide range of multi-agent frameworks, from general-purpose conversation platforms~\citep{wu2024autogen,li2023camel} to specialized systems for code generation~\citep{niu2025flow,kim2024llm} and workflow automation~\citep{zhuge2024gptswarm,wu2024agentkit}.
In this work, we focus on the common LLM-agent setting in which an orchestrator delegates subtasks to multiple LLM workers that share the objective of completing the user request.

A key advantage of multi-agent systems is the potential for \emph{parallelism}. When subtasks at the same DAG layer are independent, multiple worker agents can execute them concurrently, so that wall-clock time is determined by the slowest worker rather than the sum of all tasks. In theory, such parallel execution should provide substantial speedups, and many recent multi-agent systems are built on this premise~\citep{qian2025scaling,niu2025flow,kim2024llm,yu2025dyntaskmas}.

However, existing parallel multi-agent systems often fail to achieve the expected speedups and can even be slower than their sequential counterparts~\citep{hong2024metagpt,yao2025can,xu2026rethinking}. This suggests a fundamental limitation in current execution strategies: they lack an explicit criterion for determining when parallelization is actually worthwhile. Existing methods either rely on fixed execution policies, such as always parallelizing subtasks once explicit dependencies are satisfied~\citep{qian2025scaling,niu2025flow}, or depend on heuristic LLM-based scheduling decisions~\citep{dang2026multiagent,wang2026agentconductor}. This raises a central question: when is parallel execution truly beneficial for LLM-based multi-agent systems?

We argue that the answer depends on accounting for the hidden costs introduced by parallelization. We identify two such costs that parallelism pays but a sequential agent does not. The first is \emph{re-exploration cost}: any agent, serial or parallel, must do some exploration to understand a task before solving it, but a serial agent does so once and reuses the resulting context for free, whereas parallel workers have to independently rediscover the same conventions, assumptions, and intermediate decisions before they can begin useful work.
Note that this is not generic exploration, it is the \emph{redundant} portion that exists only because the orchestrator's planning state cannot be transferred to parallel workers for free. The second is \emph{alignment cost}: independently generated outputs must be reconciled into a coherent final result, and inconsistencies among workers may require additional correction.
These costs arise because LLM-generated DAGs typically capture explicit dependencies among subtasks but overlook implicit dependencies, such as shared naming conventions, formatting choices, or global design decisions, that a single agent would resolve incrementally. Parallel execution is therefore beneficial only when the time saved by concurrency outweighs the additional re-exploration and alignment costs. Figure~\ref{fig:motivation} illustrates this phenomenon on a representative task.

\begin{figure}[t]
  \centering
  \includegraphics[width=\linewidth]{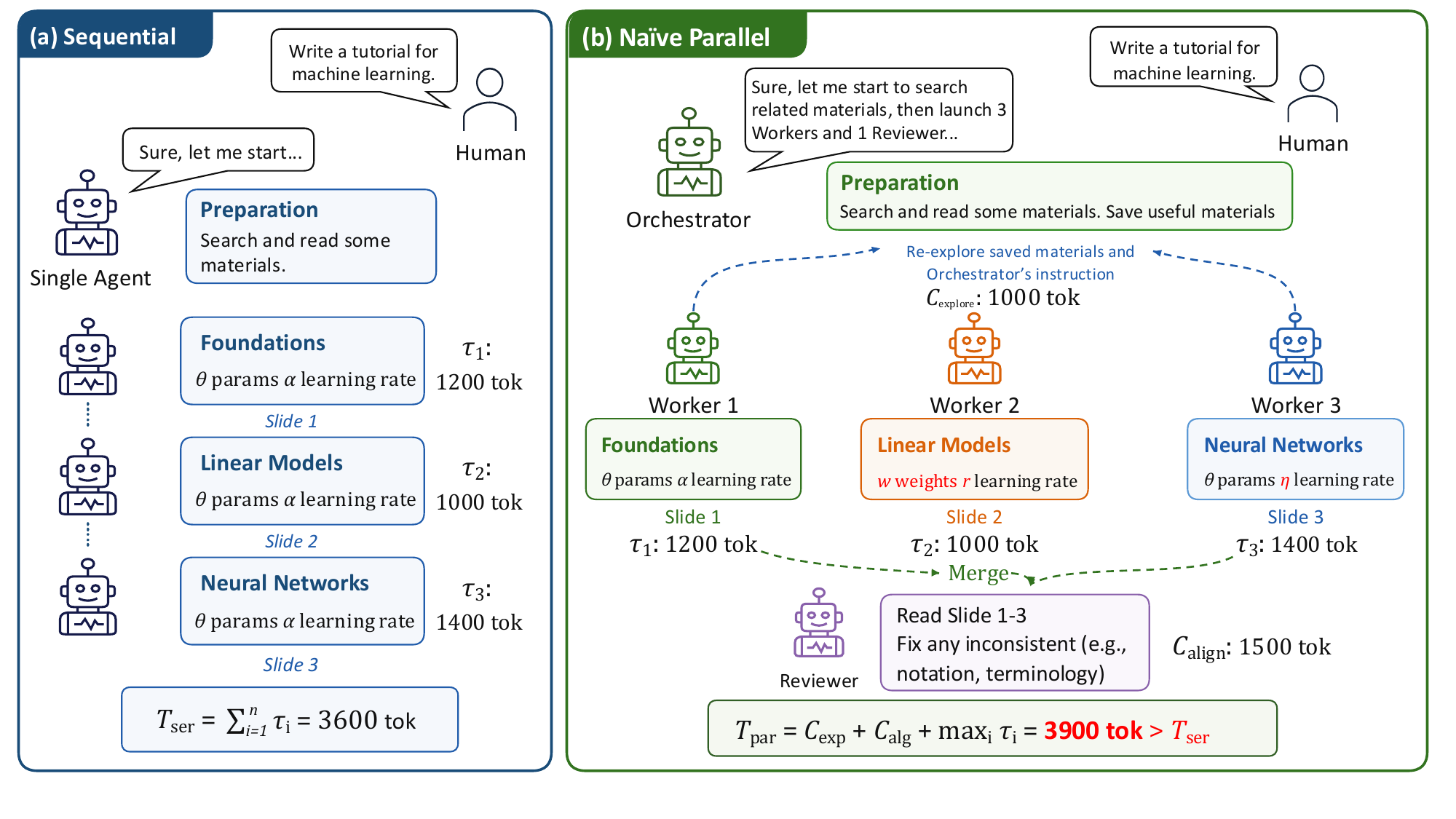}
\vspace{-25px}
  \caption{Comparison of two execution strategies on a task with four subtasks. (a)~Sequential: a single agent executes all tasks one by one with no overhead. (b)~Naive parallel: workers run concurrently but incur re-exploration cost (orange: orchestrator planning plus per-worker rediscovery of shared context) and alignment cost (red: fixing conflicts between independently produced outputs), making the total time exceed the sequential baseline.}
  \label{fig:motivation}
  \vspace{-10px}
\end{figure}

In this paper, we instead propose a parallelization criterion that makes this trade-off explicit and computable before execution. For a group of subtasks, the estimated serial cost is the sum of all subtask costs, while the estimated parallel cost consists of the slowest worker's execution cost together with the re-exploration and alignment costs. Parallel execution is selected only when the estimated parallel cost is strictly lower than the estimated serial cost. The natural unit for this criterion is wall-clock time, which is what users observe. However, predicting wall-clock time at planning time turns out to be unreliable for two reasons. First, it is backend-dependent (server load, batching, retries). Second, and more fundamentally, LLMs are systematically miscalibrated when asked to estimate their own task duration: they anchor on human-engineering norms (``this looks like a few-day task'') rather than on their own throughput, and their estimates exhibit substantial rank reversals relative to realized execution time, so even the relative ordering of task costs can flip between the estimate and reality. We therefore measure cost in predicted output tokens, a quantity an LLM can estimate well because it is a property of its own next response (the size of the artifact it would write) rather than the duration of an agentic process. For a fixed model and decoding configuration, tokens are emitted at an approximately constant rate, providing a backend-independent proxy for time. Unlike fixed or heuristic policies, this token-cost criterion is an explicit, layer-wise decision rule over the task DAG.

Based on this criterion, we propose SquidAgent, a cost-aware parallel multi-agent framework that realizes token-cost scheduling in practice. SquidAgent first decomposes a user request into a DAG and, in the same planning response, estimates token budgets for both subtask execution and layer-level alignment cost. It then applies the criterion to each topological layer, choosing parallel execution when the estimated parallel cost is lower than the estimated serial cost, and serial execution otherwise. To further reduce the hidden costs themselves, workers fork the orchestrator's conversation, inheriting global conventions and prior decisions. This reduces redundant reconstruction of the planning context, motivating the operational approximation $C_{\mathrm{exp}}\approx0$. The orchestrator also writes a shared convention block before each parallel layer, reducing subsequent inconsistency and making the upfront planning cost estimable. Code is available at: \url{https://github.com/tmllab/2026_NeurIPS_SquidAgent}.

Our main contributions are summarized as follows:
\begin{itemize}[leftmargin=*]

\item \textbf{Hidden-cost analysis of parallel multi-agent execution.}
We identify two operationally distinct sources of coordination overhead in parallel LLM-agent execution: (i) the \emph{re-exploration cost}, where parallel workers redundantly reconstruct the orchestrator’s planning context that a serial agent retains implicitly, and (ii) the \emph{alignment cost}, incurred when independently generated outputs must be reconciled into a globally consistent solution. We show that parallel execution can be counterproductive when these costs outweigh the benefit of concurrency.

\item \textbf{A token-cost criterion for adaptive parallelization.}
We observe that wall-clock time is the natural optimization target, but is poorly estimated by LLMs due to systematic miscalibration in predicting task duration. We instead use predicted output tokens as a proxy cost, which LLMs can estimate more reliably than wall-clock time as a property of their own responses and which is backend-independent. Based on this, we formalize parallelization as a cost-sensitive scheduling problem and derive a layer-wise criterion comparing (i) predicted serial cost (sum of subtasks) against (ii) predicted parallel cost (critical path plus re-exploration and alignment costs). The criterion is computed before execution and applied per topological layer of the task DAG.

\item \textbf{The SquidAgent framework.}
We instantiate this criterion in SquidAgent, a parallel multi-agent framework for cost-aware execution. Our method (i) estimates token budgets in a single planning pass, (ii) initializes workers from a shared orchestrator state to reduce redundant re-exploration, and (iii) writes a shared convention block to reduce alignment cost across parallel workers.
Across all nine evaluation tasks, SquidAgent achieves a 2.6$\times$ average wall-time speedup and a 2.2$\times$ average throughput improvement over Claude Code. It also achieves a 2.0$\times$ average throughput improvement over the strongest multi-agent baseline while obtaining the best overall task completion quality.
\end{itemize}

\section{Related Work}
\label{sec:related}

\paragraph{Multi-agent LLM systems.}
Multi-agent LLM systems decompose complex tasks into subtasks and solve them through coordinated agents~\citep{guo2024llmma}. 
Existing systems differ mainly in how they organize execution. 
Sequential pipelines, such as MetaGPT~\citep{hong2024metagpt}, SeqCV~\citep{yao2025can}, and OneFlow~\citep{xu2026rethinking}, execute subtasks step by step and preserve global consistency, but they sacrifice parallel speedup. 
Parallel systems, such as MacNet~\citep{qian2025scaling}, Flow~\citep{niu2025flow}, and LLMCompiler~\citep{kim2024llm}, dispatch independent subtasks to concurrent workers, but typically rely on fixed parallelization rules rather than an explicit cost-benefit criterion. 
Other frameworks, including AutoGen~\citep{wu2024autogen}, ChatDev~\citep{qian2024chatdev}, CAMEL~\citep{li2023camel}, and AgentVerse~\citep{chen2024agentverse}, provide flexible agent composition, but leave the execution policy largely to the developer. 
Debate-style systems~\citep{du2024improving,liang2024encouraging,chan2024chateval,chen2024reconcile} improve answer quality through iterative critique and consensus, which is orthogonal to our focus on execution scheduling.
{Consensus Matrix explores role-specialized multi-agent collaboration for structured decision-making in visual media workflows~\citep{zhang2026consensus}.}
Heterogeneous agent collaborative reinforcement learning shares verified rollouts during training while operating independently at inference time~\citep{zhang2026heterogeneous}. In contrast, we study how multiple workers should be coordinated during inference.

\paragraph{Graph-based task decomposition.}
Graph-structured reasoning and execution have been widely used to expose dependencies among intermediate steps. 
Tree-of-Thoughts~\citep{yao2023tot}, Graph-of-Thoughts~\citep{besta2024graph}, and Skeleton-of-Thought~\citep{ning2024skeletonofthought} structure single-agent reasoning as trees, graphs, or parallel skeleton expansion. 
In multi-agent settings, GPTSwarm~\citep{zhuge2024gptswarm}, AgentKit~\citep{wu2024agentkit}, TDAG~\citep{wang2025tdag}, and DynTaskMAS~\citep{yu2025dyntaskmas} represent tasks or agents as computational graphs with explicit dependencies. 
These graph representations naturally induce dependency layers that can serve as synchronization points. 
Our work uses this layer structure not merely for execution, but as the unit for deciding whether parallelization is worthwhile.

\paragraph{Scheduling and parallelization decisions.}
Existing systems typically use fixed serial execution, fixed parallel execution, or heuristic adaptive control. 
Fixed strategies ignore task-dependent coordination costs, while adaptive systems such as the Evolving Orchestrator~\citep{dang2026multiagent} and AgentConductor~\citep{wang2026agentconductor} adjust execution strategies without an explicit serial-versus-parallel cost criterion.
Adaptive reasoning termination~\citep{huang2026does} reduces computation within individual responses.
Related work on token efficiency~\citep{lin2026stop}, agent scheduling theory~\citep{wei2026agent}, and workflow optimization~\citep{yue2026static} highlights the importance of reducing agent execution cost, but does not provide an operational decision rule for LLM-agent parallelization. 
SquidAgent addresses this gap by comparing the estimated serial cost of each dependency layer with its effective parallel cost, including coordination overhead.

\section{Method}
\label{sec:method}

In this section, we present SquidAgent. 
We first define the task decomposition setting and notation (Section~\ref{sec:formulation}). 
We then derive a wall-clock parallelization criterion, explain why it is unreliable to estimate before execution, and introduce a token-cost criterion as a computable alternative (Section~\ref{sec:criterion}). 
Finally, we describe how SquidAgent implements this criterion through context forking, upfront convention planning, and deterministic layer-wise scheduling (Section~\ref{sec:system}).

\subsection{Setting and notation}
\label{sec:formulation}

Given a user request $\mathcal{R}$, an \emph{orchestrator} $\mathcal{O}$ decomposes $\mathcal{R}$ into a directed acyclic graph (DAG) $\mathcal{G}$ over subtasks $t_1,\ldots,t_n$, where edges encode dependencies. 
In the same planning step, the orchestrator also produces the cost estimates consumed by the scheduler:
\begin{equation}
  \bigl(\mathcal{G},\; \{\tau_i\}_{i=1}^{n},\; \{\widehat{C}_{\mathrm{alg}}(\mathcal{L}_k)\}_{k=1}^{L}\bigr)
  \;\sim\;
  \mathcal{O}(\mathcal{R}),
  \label{eq:decompose}
\end{equation}
where $\tau_i$ estimates the cost of completing subtask $t_i$, and $\widehat{C}_{\mathrm{alg}}(\mathcal{L}_k)$ estimates the alignment cost for layer $\mathcal{L}_k$, i.e., the effort needed to resolve inconsistencies among parallel outputs such as notation, terminology, interfaces, formatting, or cross-references.

The DAG is partitioned into topological layers $\mathcal{L}_1,\ldots,\mathcal{L}_L$, where every dependency of a task in $\mathcal{L}_k$ is satisfied by earlier layers. 
Tasks within the same layer are therefore eligible for parallel execution. 
For each layer, the scheduler chooses an execution mode $m_k \in \{\textsc{Serial}, \textsc{Parallel}\}$. 
Under \textsc{Serial}, the orchestrator completes the subtasks one after another; under \textsc{Parallel}, separate workers execute them concurrently.

\subsection{Cost decomposition and parallelization criterion}
\label{sec:criterion}

We first formulate the layer-wise cost in wall-clock time, and then replace it with a more tractable token-based proxy.

\paragraph{Layer cost in wall-clock time.}
Let $\tilde{\tau}_i$ denote the wall-clock time required to complete subtask $t_i$. 
Under serial execution, the layer cost is the sum of all subtask times:
\begin{equation}
  \mathrm{Time}_{\mathrm{ser}}(\mathcal{L}_k)
  \;\triangleq\;
  \sum_{i\in\mathcal{L}_k}\tilde{\tau}_i.
  \label{eq:time_serial}
\end{equation}
Under parallel execution, workers execute subtasks concurrently, so the execution cost is determined by the slowest worker. 
However, parallelism also introduces two additional costs. 
First, independent workers may need to rediscover planning context that the orchestrator has already established, incurring re-exploration time $\widetilde{C}_{\mathrm{exp}}(\mathcal{L}_k)$. 
Second, independently produced outputs may require reconciliation, incurring alignment time $\widetilde{C}_{\mathrm{alg}}(\mathcal{L}_k)$. 
Thus the parallel cost is
\begin{equation}
  \mathrm{Time}_{\mathrm{par}}(\mathcal{L}_k)
  \;\triangleq\;
  \underbrace{\max_{i\in\mathcal{L}_k}\tilde{\tau}_i}_{\text{critical path}}
  \;+\;
  \underbrace{\widetilde{C}_{\mathrm{exp}}(\mathcal{L}_k)}_{\text{re-exploration}}
  \;+\;
  \underbrace{\widetilde{C}_{\mathrm{alg}}(\mathcal{L}_k)}_{\text{alignment}}.
  \label{eq:time_parallel}
\end{equation}
A layer should be parallelized only when doing so reduces wall-clock time:
\begin{equation}
  m_k = \textsc{Parallel}
  \quad\Longleftrightarrow\quad
  \mathrm{Time}_{\mathrm{ser}}(\mathcal{L}_k)
  \;>\;
  \mathrm{Time}_{\mathrm{par}}(\mathcal{L}_k).
  \label{eq:time_criterion}
\end{equation}

\paragraph{Why estimating wall-clock time fails.}
Although Eq.~\ref{eq:time_criterion} is the desired decision rule, its terms cannot be reliably predicted before execution. 
Wall-clock time depends on backend conditions such as server load, batching, throttling, retries, and tool-call latency. 
LLMs are also poorly calibrated when asked to estimate their own execution duration: they tend to answer in human-engineering time rather than model-runtime time.
Such estimates may mis-rank subtasks relative to their realized execution times, since human-difficult tasks are not always slow for LLM agents, while long outputs, tool calls, and retries can dominate actual runtime.

\paragraph{From wall-clock time to output-token budgets.}
We address this issue by estimating output length rather than execution time.
Instead of asking ``how long will this task take?'', SquidAgent asks ``how many output tokens will this task require?'' This quantity is better calibrated because it reflects the size of the artifact the LLM will generate, rather than the duration of an agentic process. For a fixed model and decoding configuration, output tokens are produced at a roughly stable rate, making token count a practical proxy for relative execution time. Moreover, token budgets can be grounded in concrete deliverables, turning estimation into a content-length prediction rather than a runtime prediction.
This proxy is primarily suited to LLM-generation-dominated workflows and does not directly capture latency from tool execution, retrieval, or external API calls.

\paragraph{Token-cost criterion.}
We therefore measure every quantity in Eq.~\ref{eq:time_criterion} using predicted output tokens. 
Let $\tau_i$ be the predicted output-token budget for subtask $t_i$, and let $C_{\mathrm{exp}}(\mathcal{L}_k)$ and $C_{\mathrm{alg}}(\mathcal{L}_k)$ denote the token-valued re-exploration and alignment costs of layer $\mathcal{L}_k$. 
The serial and parallel token costs are
\begin{equation}
  T_{\mathrm{ser}}(\mathcal{L}_k)
  \;\triangleq\;
  \sum_{i\in\mathcal{L}_k}\tau_i,
  \label{eq:serial}
\end{equation}
\begin{equation}
  T_{\mathrm{par}}(\mathcal{L}_k)
  \;\triangleq\;
  \underbrace{\max_{i\in\mathcal{L}_k}\tau_i}_{\text{critical path}}
  \;+\;
  \underbrace{C_{\mathrm{exp}}(\mathcal{L}_k)}_{\text{re-exploration}}
  \;+\;
  \underbrace{C_{\mathrm{alg}}(\mathcal{L}_k)}_{\text{alignment}}.
  \label{eq:parallel_full}
\end{equation}
The decision rule becomes
\begin{equation}
  m_k = \textsc{Parallel}
  \quad\Longleftrightarrow\quad
  T_{\mathrm{ser}}(\mathcal{L}_k)
  \;>\;
  T_{\mathrm{par}}(\mathcal{L}_k).
  \label{eq:criterion_full}
\end{equation}
This criterion makes explicit why always parallelizing independent tasks can be suboptimal: the savings from replacing a sum with a maximum must exceed the additional re-exploration and alignment costs. 
Conversely, always-serial execution will miss potential speedups on layers where the parallel cost is lower than the serial cost.

\paragraph{Reducing the criterion to an operational form.}
Although token budgets are easier to estimate than wall-clock time, not all terms in Eq.~\ref{eq:criterion_full} are equally tractable. 
The per-task budgets $\tau_i$ can be estimated during decomposition, whereas $C_{\mathrm{exp}}(\mathcal{L}_k)$ and $C_{\mathrm{alg}}(\mathcal{L}_k)$ depend on emergent worker behavior in a naive parallel system.

SquidAgent reduces these difficult terms by design. 
First, workers fork the orchestrator's session and inherit prior reasoning, decisions, constraints, and partial plans. 
This shared context reduces duplicated re-exploration, allowing us to approximate $C_{\mathrm{exp}}(\mathcal{L}_k)\approx 0$. 
Second, before parallel execution, the orchestrator writes a layer-specific convention block that specifies shared assumptions, interfaces, formatting rules, and coordination constraints. 
This moves part of the alignment burden from post-hoc reconciliation to upfront planning, whose estimated cost is denoted by $\widehat{C}_{\mathrm{alg}}(\mathcal{L}_k)$. 
The effective parallel cost is therefore
\begin{equation}
  \widetilde{T}_{\mathrm{par}}(\mathcal{L}_k)
  \;\triangleq\;
  \max_{i\in\mathcal{L}_k}\tau_i
  \;+\;
  \widehat{C}_{\mathrm{alg}}(\mathcal{L}_k).
  \label{eq:parallel_effective}
\end{equation}

We then define the estimated layer-wise cost ratio as
\begin{equation}
  \widehat{\rho}_k
  \;\triangleq\;
  \frac{T_{\mathrm{ser}}(\mathcal{L}_k)}
       {\widetilde{T}_{\mathrm{par}}(\mathcal{L}_k)}.
  \label{eq:cost_ratio}
\end{equation}
A ratio above one indicates that parallel execution is predicted to be cheaper than serial execution. 
However, since $\{\tau_i\}$ and $\widehat{C}_{\mathrm{alg}}(\mathcal{L}_k)$ are noisy estimates rather than exact costs, deciding purely by $\widehat{\rho}_k>1$ can be vulnerable. 
SquidAgent therefore uses a safety margin $\alpha>1$:
\begin{equation}
  m_k = \textsc{Parallel}
  \quad\Longleftrightarrow\quad
  \widehat{\rho}_k > \alpha.
  \label{eq:margin}
\end{equation}
When $\widehat{\rho}_k\!\in\!(1,\alpha]$, the predicted advantage of parallel execution may be erased by estimation noise, so the layer is executed serially.

\paragraph{Robustness to ratio-estimation error.}
Let $\rho_k$ denote the true serial-to-parallel cost ratio. If $|\widehat{\rho}_k-\rho_k|\le\varepsilon$, then $\rho_k>\alpha+\varepsilon$ guarantees a parallel decision, while $\rho_k<\alpha-\varepsilon$ guarantees a serial decision. Estimation error can change the threshold decision only within $\varepsilon$ of $\alpha$. In particular, $\varepsilon<\alpha-1$ prevents incorrect parallelization of any layer with $\rho_k\le1$. The margin can nevertheless forgo modest parallel gains. Proposition~\ref{prop:ratio_robustness} in Appendix~\ref{app:robustness} gives the full statement and proof.

\begin{algorithm}[t]
\caption{SquidAgent: Layer-wise Scheduling and Execution}
\label{alg:costpar}
\KwIn{User request $\mathcal{R}$, safety margin $\alpha$}
\KwOut{Completed deliverables}
$\bigl(\mathcal{G}, \{\tau_i\}_{i=1}^{n}, 
\{\widehat{C}_{\mathrm{alg}}(\mathcal{L}_k)\}_{k=1}^{L}\bigr)
\sim \mathcal{O}(\mathcal{R})$\tcp*[r]{Decompose request and estimate costs}
Partition $\mathcal{G}$ into topological layers 
$\mathcal{L}_1,\ldots,\mathcal{L}_L$\tcp*[r]{Build dependency layers}
\For{$k = 1$ \KwTo $L$}{
  $\widehat{\rho}_k \gets
  \frac{\sum_{t_i\in\mathcal{L}_k}\tau_i}
  {\max_{t_i\in\mathcal{L}_k}\tau_i+
  \widehat{C}_{\mathrm{alg}}(\mathcal{L}_k)}$\tcp*[r]{Estimate cost ratio}
  \uIf{$\widehat{\rho}_k > \alpha$}{
    $m_k \gets \textsc{Parallel}$\tcp*[r]{Ratio exceeds margin}
    $\mathcal{O}$ writes conventions for $\mathcal{L}_k$\tcp*[r]{Front-load alignment}
    Workers execute $\{t_i:t_i\in\mathcal{L}_k\}$ with inherited context\tcp*[r]{Sharing context}
  }
  \Else{
    $m_k \gets \textsc{Serial}$\tcp*[r]{Margin not met}
    $\mathcal{O}$ executes $\{t_i:t_i\in\mathcal{L}_k\}$ sequentially\tcp*[r]{Preserve consistency}
  }
}
\Return completed deliverables\;
\end{algorithm}

\subsection{The SquidAgent framework}
\label{sec:system}

SquidAgent realizes Eq.~\ref{eq:margin} through three components: an orchestrator that produces both the task graph and the required token budgets, workers that inherit the orchestrator's context, and a deterministic scheduler that applies the criterion to each layer.

\paragraph{Orchestrator: planning with cost estimates.}
The orchestrator is prompted to decompose the request into a DAG and, in the same response, estimate the output-token cost of each subtask and the alignment cost of each dependency layer. 
Concretely, the structured output contains task-level fields such as
\texttt{\{id, depends\_on, estimated\_output\_tokens\}} and a layer-level field
\texttt{layer\_alignment\_tokens}. 
The field \texttt{estimated\_output\_tokens} corresponds to $\tau_i$, while \texttt{layer\_alignment\_tokens} provides $\widehat{C}_{\mathrm{alg}}(\mathcal{L}_k)$ for each layer. 
Because these estimates are produced together with the decomposition, SquidAgent does not require additional LLM calls for cost estimation. 
The alignment budget is estimated per layer rather than globally, since different layers can require very different amounts of shared convention planning. 
The complete prompt and schema are provided in Appendix~\ref{sec:prompts}.

\paragraph{Workers: context inheritance.}
A naive worker must reconstruct the global plan before doing useful work, which introduces duplicated exploration. 
SquidAgent avoids this by launching each worker from the orchestrator's session state, so the worker inherits the planning context, conventions, and architectural decisions already established by the orchestrator. 
It can therefore proceed directly to its assigned subtask rather than rediscovering global context. 
Context inheritance therefore reduces re-exploration cost to near zero, yielding $C_{\mathrm{exp}}(\mathcal{L}_k)\approx0$ in Eq.~\ref{eq:parallel_effective}. The component ablation in Section~\ref{sec:components} evaluates the contribution of this mechanism.

\paragraph{Upfront convention planning.}
Alignment cost is difficult to predict when it is left until after workers finish, because the inconsistencies among independently produced outputs are only revealed post hoc. 
SquidAgent reduces this uncertainty by specifying the shared conventions before parallel execution. 
For each layer considered for parallelization, the orchestrator writes a layer-specific policy that fixes notation, naming, interfaces, formatting rules, and cross-reference requirements. 
Workers inherit this policy through the same session fork, so many potential inconsistencies are constrained before they arise rather than repaired after completion. 
Upfront conventions reduce the need for post-hoc reconciliation, making alignment cost estimable before execution through the output-token cost of upfront convention planning. The operational criterion uses this estimate, $\widehat{C}_{\mathrm{alg}}(\mathcal{L}_k)$, as an approximation to alignment overhead.

\paragraph{Layer-wise scheduling.}
The scheduler converts the planned token budgets into execution decisions for each dependency layer. 
For a layer $\mathcal{L}_k$, it computes $\widehat{\rho}_k$ using Eq.~\ref{eq:margin} and selects \textsc{Parallel} only when $\widehat{\rho}_k>\alpha$; otherwise, it uses \textsc{Serial}. 
This applies the same cost criterion consistently across layers and avoids relying on an additional LLM call for the final scheduling decision. 
The full procedure is summarized in Algorithm~\ref{alg:costpar}.

\section{Experiments}
\label{sec:experiments}

We evaluate SquidAgent on 9 evaluation tasks spanning code generation, document authoring, and structured planning, and compare it with 7 baselines. We additionally isolate the three components (Section~\ref{sec:components}) and compare token and time prediction (Section~\ref{sec:prediction}). Additional experiments on margin stability across all nine tasks and transfer to external tasks are presented in Appendices~\ref{app:margin_stability} and~\ref{app:flow_transfer}, respectively. We describe the experimental setup (Section~\ref{sec:setup}), report throughput results (Section~\ref{sec:main_results}), evaluate output quality (Section~\ref{sec:quality}), and analyze the scheduler's sensitivity to ratio estimation and the margin parameter $\alpha$ (Section~\ref{sec:ablation}). Layer-wise scheduler decisions are detailed in Appendix~\ref{sec:analysis}. Visualizations of the results generated by all methods are provided at: \url{https://yexionglin.github.io/SquidAgent_Demo}.

\subsection{Setup}
\label{sec:setup}

\paragraph{Model.}
All methods use Claude Sonnet (\texttt{claude-sonnet-4-6}) with thinking effort set to \texttt{medium}. Unless otherwise specified, all methods are run with the same tool-access setting and the same task specification.

\paragraph{Evaluation tasks.}
We design a suite of 9 evaluation tasks spanning code generation, technical writing, and structured planning (Table~\ref{tab:evaluation_tasks}; full prompts in Appendix~\ref{app:evaluation_tasks}). 
The tasks are designed to expose different coordination challenges faced by multi-agent systems, including independent module parallelism, shared interface constraints, cross-file consistency, long-form generation, and dependencies across modules.

\emph{Heavy tasks} contain 8--27 files with 300--900 lines per file:
1)~\textbf{PixelCraft}, a Python/Pygame survival game testing parallel implementation of loosely coupled game systems;
2)~\textbf{ShopFlow}, a full-stack e-commerce application testing consistency across backend APIs, authentication, admin logic, and frontend pages;
3)~\textbf{CompressKit}, a compression toolkit testing shared data-format constraints across multiple algorithms;
4)~\textbf{ArcadeBox}, a collection of HTML5 Canvas games testing high parallelism across mostly independent modules;
5)~\textbf{SlideKit}, ML workshop slide decks and handouts testing long-form generation with cross-document consistency;
6)~\textbf{ClimateAnalysis}, analysis scripts and reports testing consistency between computed results and written summaries.

\emph{Medium tasks} contain 6--10 files with 150--600 lines per file:
7)~\textbf{LinAlgBook}, a linear algebra tutorial testing structured mathematical writing;
8)~\textbf{MathRef}, a discrete mathematics reference testing tightly coupled cross-references;
9)~\textbf{CloudDocs}, API documentation testing consistency between module-level references and cross-cutting guides.

\begin{table*}[t]
  \caption{Deliverable throughput (words/s) across 9 evaluation tasks. Computed as deliverable output divided by wall time. Higher is better. Bold indicates the best result per task.}
  \label{tab:throughput}
  \centering
  \small
  \setlength{\tabcolsep}{3pt}
  \begin{tabular}{l@{\hskip 4pt}cccccc@{\hskip 4pt}ccc@{\hskip 4pt}c}
    \toprule
    & \multicolumn{6}{c}{Heavy} & \multicolumn{3}{c}{Medium} & \\
    \cmidrule(lr){2-7} \cmidrule(lr){8-10}
    Method & Pixel. & Shop. & Comp. & Arcade. & Slide. & Climate. & LinAlg. & Math. & Cloud. & Overall \\
    \midrule
    SquidAgent & \textbf{21.0} & \textbf{35.1} & \textbf{16.1} & \textbf{52.6} & \textbf{47.7} & \textbf{27.7} & \textbf{46.3} & \textbf{48.1} & \textbf{48.2} & \textbf{38.1$\pm$12.8} \\
    \midrule
    Claude Code & 5.0 & 16.7 & 3.7 & 17.5 & 19.9 & 13.5 & 22.6 & 26.7 & 29.6 & 17.2$\pm$8.3 \\
    SeqCV & 5.3 & 15.6 & 4.0 & 14.0 & 18.1 & 15.8 & 21.8 & 27.9 & 15.6 & 15.3$\pm$7.0 \\
    MetaGPT & 12.3 & 16.3 & 6.9 & 8.7 & 9.7 & 11.4 & 20.2 & 22.4 & 14.3 & 13.6$\pm$4.9 \\
    AFlow & 4.2 & 8.8 & 3.7 & 14.8 & 12.7 & 9.5 & 16.8 & 18.3 & 15.4 & 11.6$\pm$5.0 \\
    Flow & 4.8 & 13.2 & 3.5 & 12.9 & 10.8 & 10.9 & 19.3 & 21.9 & 23.6 & 13.4$\pm$6.6 \\
    MacNet & 11.4 & 12.3 & 9.1 & 15.9 & 13.5 & 14.1 & 18.4 & 24.3 & 24.1 & 15.9$\pm$5.1 \\
    AgentConductor & 13.1 & 15.0 & 10.9 & 18.7 & 21.1 & 16.7 & 23.7 & 27.3 & 26.5 & 19.2$\pm$5.5 \\
    \bottomrule
  \end{tabular}
  \caption{Quality scores (\%) across 9 evaluation tasks. Each rubric item (21--37 per task) is scored binary YES/NO by Claude Opus 4.6; scores are normalized by per-task item count.}
  \label{tab:quality}
  \centering
  \small
  \setlength{\tabcolsep}{3pt}
  \begin{tabular}{l@{\hskip 4pt}cccccc@{\hskip 4pt}ccc@{\hskip 4pt}c}
    \toprule
    & \multicolumn{6}{c}{Heavy} & \multicolumn{3}{c}{Medium} & \\
    \cmidrule(lr){2-7} \cmidrule(lr){8-10}
    Method & Pixel. & Shop. & Comp. & Arcade. & Slide. & Climate. & LinAlg. & Math. & Cloud. & Overall \\
    \midrule
    SquidAgent & \textbf{100.0} & 96.7 & 95.8 & \textbf{100.0} & 95.5 & \textbf{96.2} & \textbf{100.0} & \textbf{100.0} & \textbf{100.0} & \textbf{98.2$\pm$2.1} \\
    Claude Code & \textbf{100.0} & 96.7 & \textbf{100.0} & 92.9 & \textbf{100.0} & 92.3 & \textbf{100.0} & 96.2 & 95.7 & 97.1$\pm$3.1 \\
    SeqCV & 94.6 & \textbf{100.0} & \textbf{100.0} & 92.9 & 81.8 & \textbf{96.2} & 90.5 & 92.3 & 91.3 & 93.3$\pm$5.5 \\
    MetaGPT & 97.3 & \textbf{100.0} & \textbf{100.0} & 96.4 & 31.8 & 38.5 & \textbf{100.0} & \textbf{100.0} & 95.7 & 84.4$\pm$28.0 \\
    AFlow & 97.3 & 76.7 & 95.8 & 67.9 & 27.3 & 38.5 & 95.2 & 46.2 & 73.9 & 68.7$\pm$26.2 \\
    Flow & 70.3 & 36.7 & \textbf{100.0} & 35.7 & \textbf{100.0} & 92.3 & \textbf{100.0} & \textbf{100.0} & 95.7 & 81.2$\pm$27.2 \\
    MacNet & 78.4 & 63.3 & 58.3 & 50.0 & 50.0 & 19.2 & 90.5 & 88.5 & 60.9 & 62.1$\pm$22.1 \\
    AgentConductor & 78.4 & 96.7 & 62.5 & 28.6 & \textbf{100.0} & 84.6 & \textbf{100.0} & 88.5 & 95.7 & 81.6$\pm$23.3 \\
    \bottomrule
  \end{tabular}
  \vspace{-10px}
\end{table*}

\paragraph{Baselines.}
We compare SquidAgent with seven representative baselines spanning single-agent execution, sequential verification, role-based workflows, workflow search, parallel execution, and adaptive coordination. 
\textbf{Claude Code} is a baseline using raw \texttt{claude -p} without additional orchestration instructions. 
\textbf{SeqCV}~\citep{yao2025can} performs sequential execution with cross-model verification. 
\textbf{MetaGPT}~\citep{hong2024metagpt} follows a role-based workflow with predefined agents such as product manager, architect, and engineer. 
\textbf{AFlow}~\citep{zhang2025aflow} searches over operator pipelines and executes the selected workflow. 
\textbf{Flow}~\citep{niu2025flow} uses git-based orchestration with parallel workers in isolated worktrees. 
\textbf{MacNet}~\citep{qian2025scaling} organizes agents in a DAG-style topology with draft-review interactions. 
\textbf{AgentConductor}~\citep{wang2026agentconductor} adaptively adjusts the coordination structure according to task difficulty.

\paragraph{Protocol and metrics.}
All methods receive the same task specification per task. We set $\alpha=1.9$ for SquidAgent. We report wall time (seconds, from first model invocation to final deliverable) and deliverable throughput (words/s, excluding intermediate artifacts). Output quality is evaluated separately in Section~\ref{sec:quality}.

\begin{figure}[t]
\vspace{-10px}
\centering
\begin{subfigure}[b]{0.44\linewidth}
  \centering
  \includegraphics[width=\linewidth]{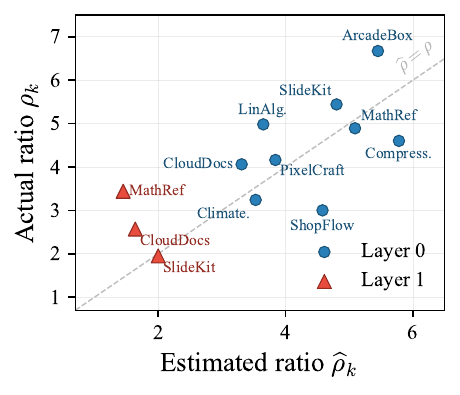}
  \vspace{-20px}
  \caption{Ratio estimation accuracy.}
  \label{fig:ratio-estimation}
\end{subfigure}
\begin{subfigure}[b]{0.44\linewidth}
  \centering
  \includegraphics[width=\linewidth]{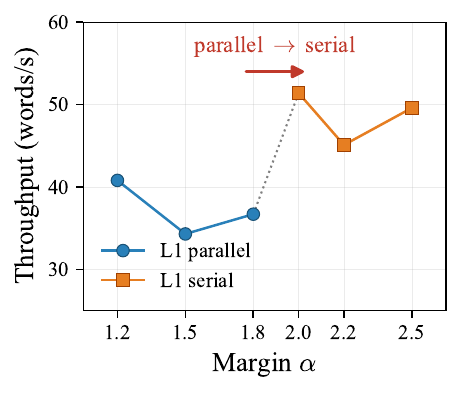}
  \vspace{-20px}
  \caption{Sensitivity to $\alpha$ (MathRef).}
  \label{fig:alpha-sensitivity}
\end{subfigure}
\caption{Scheduler sensitivity analysis.
(a)~Estimated cost ratio $\widehat{\rho}_k$ versus actual cost ratio measured after execution, across 14 multi-task dependency layers from 9 evaluation tasks. 
Layer~0 and Layer~1 correspond to the first and second dependency layers in the orchestrator-generated task DAG. 
(b)~Throughput on MathRef under different safety margins $\alpha$, with the task DAG and token estimates held fixed. 
For $\alpha \ge 2.0$, Layer~1 changes from parallel to serial execution, which improves throughput because the alignment cost of cross-referencing appendices exceeds the parallelism benefit.}
\label{fig:scheduler-analysis}
\vspace{-10px}
\end{figure}

\subsection{Throughput results}
\label{sec:main_results}

Table~\ref{tab:throughput} reports deliverable throughput on the nine evaluation tasks, measured as the number of delivered words divided by wall-clock time. 
Detailed wall-clock time and output-size statistics are provided in Appendix~\ref{app:more_results}. 
Across all tasks, SquidAgent achieves the highest throughput, while producing deliverables of comparable or larger size than the baselines.

The improvement is most evident on large tasks, where the task structure exposes substantial independent work that can be executed in parallel. 
For example, SquidAgent improves throughput over the strongest baseline by 2.8$\times$ on ArcadeBox and 2.3$\times$ on SlideKit. 
On medium tasks, the absolute gains are smaller because there are fewer independent subtasks to parallelize, but SquidAgent still achieves the best throughput on every task. 
Overall, SquidAgent increases mean throughput from 17.2 to 38.1 words/s compared with Claude Code, corresponding to a 2.2$\times$ speedup, and from 19.2 to 38.1 words/s compared with the strongest multi-agent baseline, AgentConductor, corresponding to a 2.0$\times$ speedup. 
These results show that SquidAgent improves efficiency not merely by increasing parallelism, but by parallelizing only when the estimated serial cost exceeds the effective parallel cost, including coordination overhead.

\subsection{Quality evaluation}
\label{sec:quality}

Table~\ref{tab:quality} reports output quality using task-specific binary rubrics scored by Claude Opus 4.6. 
Each rubric covers task compliance, domain-specific correctness, and overall polish, with 21--37 items per task; full rubrics are provided in Appendix~\ref{app:rubrics}.

SquidAgent achieves the highest overall score, $98.2\pm2.1$\%, reaching 100\% on 5 of 9 evaluation tasks. 
Claude Code is also strong ($97.1\pm3.1$\%) but substantially slower. 
Several multi-agent baselines show larger variance, with failures on tasks requiring complete deliverables and consistent cross-file organization. 
These results suggest that SquidAgent's throughput gains do not come at the expense of deliverable quality.

\begingroup
\subsection{Component-wise ablation}
\label{sec:components}

We individually ablate the three components on four tasks, which the orchestrator decomposes into multi-layer DAGs (Table~\ref{tab:component_ablation}). \emph{Without scheduling} parallelizes all eligible multi-task layers. \emph{Without session forking} removes inherited context. \emph{Without convention planning} leaves consistency correction to the final review. The other two components remain enabled. All three removals lower throughput; removing scheduling incurs a 32.1\% decrease. Their effects are not additive.

\begin{table}[t]
\centering
\small
\caption{Component-wise ablation: deliverable throughput (words/s) on four tasks. Higher is better.}
\label{tab:component_ablation}
\setlength{\tabcolsep}{3pt}
\begin{tabular}{lrrrrr}
\toprule
Method & Math. & Cloud. & Slide. & Climate. & Mean \\
\midrule
SquidAgent & \textbf{48.10} & \textbf{48.20} & \textbf{47.70} & \textbf{27.70} & \textbf{42.93} \\
Without scheduling & 35.54 & 31.70 & 31.37 & 17.96 & 29.14 \\
Without session forking & 41.22 & 33.02 & 30.26 & 19.04 & 30.88 \\
Without convention planning & 39.90 & 35.33 & 32.78 & 22.50 & 32.63 \\
\bottomrule
\end{tabular}
\vspace{-10px}
\end{table}

\subsection{Output-token versus wall-clock prediction}
\label{sec:prediction}

\setlength{\intextsep}{4pt}
\begin{wraptable}{r}{0.43\textwidth}
\vspace{0pt}
\captionsetup{font=small,skip=4pt}
\centering
\small
\caption{Prediction--observation correlations.}
\label{tab:prediction_correlations}
\setlength{\tabcolsep}{5pt}
\begin{tabular}{lcc}
\toprule
Prediction & $\rho$ & $r_{\log}$ \\
\midrule
Output tokens & \textbf{0.77} & \textbf{0.77} \\
Wall-clock time & 0.16 & 0.11 \\
\bottomrule
\end{tabular}
\end{wraptable}
For each of the 59 subtasks from nine tasks, we independently predict output-token count and wall-clock time five times each, and compare the median prediction for each quantity with its observed value from the recorded execution. The DAG, model, backend, and prompt structure are held fixed. Table~\ref{tab:prediction_correlations} reports Spearman's rank correlation coefficient $\rho$ between predicted and observed values, and Pearson's correlation coefficient $r_{\log}$ between their log-transformed values. Token predictions correlate more strongly with observations, supporting token-based costs in these LLM-dominated workflows. Appendix~\ref{app:prediction_details} provides additional rank-ordering analysis.
\par
\endgroup

\subsection{Scheduler sensitivity analysis}
\label{sec:ablation}

\paragraph{Ratio estimation accuracy.}
Figure~\ref{fig:scheduler-analysis}(a) compares the scheduler's estimated ratio $\hat{\rho}_k$ against the actual ratio $\rho_k$ across all 14 multi-task dependency layers. Generally, these ratios cluster near the diagonal, indicating reliable estimation. 

\paragraph{Sensitivity to $\alpha$.}
We vary $\alpha$ on MathRef. The orchestrator generate a 2-layer DAG, where Layer 0 (L0) contain 6 chapters and  Layer 1 (L1) contain 4 cross-referencing appendices. To isolate the effect of $\alpha$, we fix the task DAG, token estimates, and $\widehat{C}_{\mathrm{alg}}(\mathcal{L}_k)$ across all runs. Figure~\ref{fig:scheduler-analysis}(b) shows that at $\alpha \le 1.8$ the scheduler parallelizes L1, yielding 34--41 words/s. At $\alpha \ge 2.0$ it serializes L1, improving throughput to 45--51 words/s. The transition occurs between $\alpha=1.8$ and $\alpha=2.0$, and we set $\alpha=1.9$ as the default, which correctly serializes L1 for this tightly coupled layer while preserving parallelization for layers with higher ratios.

\section{Conclusion}

We present SquidAgent, a cost-aware scheduler for multi-agent LLM systems. 
Parallel execution can introduce re-exploration and alignment costs that offset its speedup when subtasks are tightly coupled. 
SquidAgent captures this trade-off with a token-cost criterion and decides, for each dependency layer, whether to run serially or in parallel. 
It further reduces coordination cost through context inheritance and upfront convention planning. 
Across nine evaluation tasks, SquidAgent improves mean throughput by 2.2$\times$ over Claude Code and 2.0$\times$ over the strongest multi-agent baseline, while maintaining competitive quality. 
These results highlight cost-aware scheduling as a practical principle for efficient multi-agent LLM execution.

SquidAgent has three main limitations. First, its fixed safety margin does not adapt to uncertainty in token estimates. Second, its token proxy does not capture tool or API latency. Third, SquidAgent makes a single serial-or-parallel decision for each layer, even when some tasks would benefit from parallel execution while others would be better run serially. Appendix~\ref{app:limitations} discusses these limitations and future directions.

\begin{ack}
Tongliang Liu is partially supported by the following Australian Research Council projects: DP260102466, FT220100318, DP220102121, LP220100527, LP220200949.
Zhen Fang is partially supported by the Australian Research Council project DE250100363.
Yu Yao is partially supported by the following Australian Research Council projects: DE260101993, DP260102466.
Bo Han is partially supported by RGC Research Fellow Scheme No. RFS2627-2S02, RGC Young Collaborative Research Grant No. C2005-24Y, RGC General Research Fund No. 12202026, RGC General Research Fund No. 12200725, NSFC Major Research Plan No. 92570109, NSFC General Program No. 62376235, HKBU Faculty Niche Research Areas No. RC-FNRA-IG/25-26/SCI/05, and HKBU CSD Departmental Incentive Scheme.
\end{ack}

\bibliography{references}
\bibliographystyle{plainnat}

\newpage
\appendix
\begin{flushleft}
	{\LARGE\textsc{Appendix}}
\end{flushleft}

\etocdepthtag.toc{mtappendix}
    \etocsettagdepth{mtchapter}{none}
    \etocsettagdepth{mtappendix}{subsection}
    
    {
        \hypersetup{linkcolor=black}
    	\footnotesize\tableofcontents
    }

\newpage

\section{Evaluation Tasks}
\label{app:evaluation_tasks}

Table~\ref{tab:evaluation_tasks} summarizes the nine evaluation tasks used in our experiments. 
The tasks are grouped by output scale and designed to cover different coordination requirements, including independent module generation, shared interfaces, cross-file consistency, long-form document construction, and multi-stage dependencies. 
For each task, all methods receive exactly the same user prompt. 
No additional implementation hints, file-structure constraints, decomposition guidance, or style specifications are provided.

\begin{table}[h]
  \caption{Characteristics of the evaluation tasks. Large tasks require substantial software systems or document collections, while medium tasks require smaller but still multi-part outputs with consistency constraints.}
  \label{tab:evaluation_tasks}
  \centering
  \small
  \resizebox{\textwidth}{!}{
  \begin{tabular}{llll}
    \toprule
    Task & Output type & Main coordination challenge & Scale \\
    \midrule
    PixelCraft & Python/Pygame & Independent game modules with shared runtime logic & Large \\
    ShopFlow & Flask + HTML/JS & Backend--frontend consistency and route integration & Large \\
    CompressKit & Python & Modular utilities with shared APIs and tests & Large \\
    ArcadeBox & HTML5 Canvas/JS & Multiple interactive components under one interface & Large \\
    SlideKit & HTML & Consistent slide structure, styling, and navigation & Large \\
    ClimateAnalysis & Python + Markdown & Multi-stage analysis, reporting, and cross-reference consistency & Large \\
    \midrule
    LinAlgBook & Markdown & Long-form chapter consistency and notation reuse & Medium \\
    MathRef & Markdown & Cross-referenced appendices and shared mathematical notation & Medium \\
    CloudDocs & Markdown & Structured technical documentation with consistent terminology & Medium \\
    \bottomrule
  \end{tabular}
  }
\end{table}

We provide the complete user prompt for each evaluation task below. 
Each prompt is used as the sole input to every method, ensuring that differences in performance come from the execution strategy rather than from task-specific guidance.

\subsubsection*{PixelCraft}
\begin{tcolorbox}[colback=gray!5, colframe=gray!60, boxrule=0.5pt, breakable]
{\small
Build a 2D pixel survival game inspired by Minecraft using Python and Pygame. The game should be a side-scrolling 2D platformer view where the player can explore, mine blocks, collect resources, craft items, fight enemies, and survive through day-night cycles.

\textbf{Core Features:}
\textit{World:} Procedurally generated open world (at least 256$\times$256 blocks), multiple biomes (forest, desert, snow, plains), underground caves with ore deposits, natural structures.
\textit{Blocks and Resources:} At least 15 block types, each with visual appearance, mining hardness, and drop items.
\textit{Player:} Movement with jump and gravity physics, collision detection, health/hunger systems, mining, placing blocks, attacking enemies.
\textit{Inventory and Crafting:} 36-slot inventory with hotbar, drag-and-drop, 3$\times$3 crafting grid with 15+ recipes.
\textit{Enemies:} 3+ enemy types with basic AI (chase, attack, wander, flee).
\textit{Day-Night Cycle:} Full cycle over 10 minutes, sky transitions, darkness overlay, torch lighting.
\textit{UI:} Health/hunger bars, hotbar, minimap, debug info, inventory overlay.
\textit{Save/Load:} JSON save with compression, multiple slots, auto-save.

\textbf{Constraints:} Python 3.10+ with Pygame only. No image/sound assets---use colored rectangles. 1024$\times$768 window, 60 FPS. Single entry point script.
}
\end{tcolorbox}

\subsubsection*{ShopFlow}
\begin{tcolorbox}[colback=gray!5, colframe=gray!60, boxrule=0.5pt, breakable]
{\small
Build a full-stack e-commerce platform with a Python backend and web frontend. Complete online store with product browsing, shopping cart, orders, and admin dashboard.

\textbf{Core Features:}
\textit{Product Catalog:} Browse with pagination, filter by category/price/rating, search, sort, product detail pages with reviews.
\textit{Shopping Cart:} Add/update/remove items, subtotals, persist across navigation, stock checking.
\textit{User Accounts:} Registration, login/logout with sessions, order history, product reviews.
\textit{Checkout and Orders:} Shipping address form, order confirmation, payment simulation, status tracking, cancellation.
\textit{Admin Dashboard:} Statistics overview, order management, product CRUD, sales analytics.
\textit{Frontend:} Responsive design, navigation with search bar and cart icon, all data loaded dynamically from API.

\textbf{Technical Requirements:} Flask + SQLAlchemy (SQLite). Vanilla HTML/CSS/JS only. JSON REST APIs with fetch(). Session-based auth. Database seed script with 20+ products. Single command launch (\texttt{python app.py}).

\textbf{Quality:} Production-quality error handling, input validation, security (password hashing, CSRF-safe sessions). Fully functional frontend. Working admin analytics with CSS-based charts. 600+ lines of CSS.
}
\end{tcolorbox}

\subsubsection*{CompressKit}
\begin{tcolorbox}[colback=gray!5, colframe=gray!60, boxrule=0.5pt, breakable]
{\small
Build a multi-algorithm file compression toolkit in Python implementing several compression algorithms from scratch.

\textbf{Algorithms} (all from scratch, no external compression libraries): Huffman coding, LZ77, LZW, Burrows-Wheeler Transform with Move-to-Front encoding, Run-Length Encoding, Arithmetic coding, Deflate (LZ77 + Huffman). Each must support compress/decompress with round-trip correctness.

\textbf{Archive Format:} Custom format for bundling multiple files. Store filenames, sizes, checksums. Support directories. Integrity verification on extraction.

\textbf{Data Integrity:} CRC32 checksum (lookup table), Adler32 rolling checksum, verification on decompression.

\textbf{CLI:} Compress/decompress single files, create/extract archives, benchmark mode (all algorithms on a file with compression ratio, speed, correctness), show metadata.

\textbf{Self-Test:} Generate various data patterns and verify all algorithms produce correct round-trip results.

\textbf{Constraints:} Python 3.10+, no external libraries. Handle edge cases (empty input, single byte, identical bytes, random bytes). Bit-level I/O utility.
}
\end{tcolorbox}

\subsubsection*{ArcadeBox}
\begin{tcolorbox}[colback=gray!5, colframe=gray!60, boxrule=0.5pt, breakable]
{\small
Build a browser-based game arcade with 8 playable games. Main menu for game selection, HTML5 Canvas rendering, persistent high scores in localStorage, global statistics.

\textbf{Games:} (1)~Tetris: 10$\times$20 grid, 7 tetrominoes, SRS rotation, wall kicks, ghost piece, scoring 100/300/500/800. (2)~Minesweeper: 3 difficulties, flood-fill reveal, first-click safe. (3)~2048: 4$\times$4 grid, smooth animations, 90\%/10\% tile spawn. (4)~Snake: growing snake, increasing speed. (5)~Breakout: paddle/ball/bricks, AABB collision, 3 lives. (6)~Sudoku: puzzle generator via backtracking, 3 difficulties, pencil marks, conflict highlighting. (7)~Memory Match: 3 grid sizes, flip animation, move/time tracking. (8)~Gomoku: 15$\times$15 board, minimax AI with alpha-beta pruning (depth 3--4).

\textbf{Requirements:} Pure HTML/CSS/JS, no frameworks or external assets. Canvas rendering at 60 FPS. Pause, restart, back-to-menu for each game. Each game file 300+ lines. Works by opening \texttt{index.html} directly.
}
\end{tcolorbox}

\subsubsection*{SlideKit}
\begin{tcolorbox}[colback=gray!5, colframe=gray!60, boxrule=0.5pt, breakable]
{\small
Create HTML slide presentations for a 2-day ML workshop. Each presentation is a standalone HTML file with inline CSS/JS and slide navigation.

\textbf{Main presentations} (20--30 slides, 500--700 lines each): (1)~\texttt{day1\_foundations.html}: ML foundations---supervised/unsupervised, bias-variance, cross-validation, metrics. (2)~\texttt{day1\_linear.html}: Linear models---regression, logistic regression, regularization, SVMs. (3)~\texttt{day2\_trees.html}: Tree methods---decision trees, random forests, gradient boosting. (4)~\texttt{day2\_neural.html}: Neural networks---backpropagation, CNNs, RNNs, transfer learning.

\textbf{Summary materials} (reference ALL presentations, must use identical notation): (5)~\texttt{quick\_reference.html}: Key formulas citing exact slide numbers, ``Practice: Exercise X.Y'' pointers. (6)~\texttt{exercise\_booklet.html}: 12 exercises (3 per session), numbered X.Y, ``See: \S X'' pointers, solutions with exact source notation.

\textbf{Master index} (depends on all files): (7)~\texttt{workshop\_index.html}: Every technical term with definition, links to quick-reference and exercises.

All 7 files share the same CSS color scheme and layout conventions.
}
\end{tcolorbox}

\subsubsection*{ClimateAnalysis}
\begin{tcolorbox}[colback=gray!5, colframe=gray!60, boxrule=0.5pt, breakable]
{\small
Create a data analysis project investigating global climate trends. Analysis scripts and summary reports.

\textbf{Analysis scripts} (standalone Python, pure Python with no pandas/numpy): (1)~\texttt{temperature\_analysis.py}: Temperature anomalies, annual means, moving averages, linear regression from scratch, seasonal decomposition. (2)~\texttt{precipitation\_analysis.py}: 6-region precipitation, drought detection, inter-region correlation, Mann-Kendall trend test from scratch. (3)~\texttt{co2\_analysis.py}: CO2 concentration, growth rates, seasonal cycle, exponential curve fitting, projections. (4)~\texttt{sea\_level\_analysis.py}: 5-station tide gauge data, station trends, acceleration detection (quadratic vs linear), confidence intervals.

(5)~\texttt{correlation\_analysis.py}: Cross-variable correlations importing data from all 4 scripts above, Pearson/Spearman correlations, lag analysis, Granger causality from scratch.

\textbf{Summary reports} (every number must match script outputs): (6)~\texttt{executive\_summary.md}: Key findings with exact values. (7)~\texttt{methodology.md}: Statistical methods with exact formulas as implemented. (8)~\texttt{data\_dictionary.md}: All variables, column names, units, ranges matching scripts exactly.

Consistent variable naming, ISO 8601 dates, standard units throughout.
}
\end{tcolorbox}

\subsubsection*{LinAlgBook}
\begin{tcolorbox}[colback=gray!5, colframe=gray!60, boxrule=0.5pt, breakable]
{\small
Write a beginner-friendly linear algebra tutorial for first-year undergraduates. Six markdown files:

(1)~\texttt{vectors.md}: Vector spaces, operations, linear independence, spanning sets. (2)~\texttt{matrix\_ops.md}: Matrix multiplication, inverse, transpose, row reduction. (3)~\texttt{determinants.md}: Determinant definition, properties, cofactor expansion, Cramer's rule. (4)~\texttt{eigen.md}: Eigenvalues, eigenvectors, diagonalization, characteristic polynomial. (5)~\texttt{svd.md}: Singular value decomposition, low-rank approximation, applications. (6)~\texttt{pca.md}: Principal component analysis, covariance matrix, dimensionality reduction.

Each file: 800--1200 words, 2+ LaTeX worked examples. Cross-chapter notation consistency is critical: vector notation, default orientation, subscript start, transpose symbol, norm brackets, matrix size convention, eigenvalue ordering.
}
\end{tcolorbox}

\subsubsection*{MathRef}
\begin{tcolorbox}[colback=gray!5, colframe=gray!60, boxrule=0.5pt, breakable]
{\small
Write a comprehensive discrete mathematics textbook for undergraduates. Six chapter files (3000--4000 words, 5+ LaTeX proofs each): propositional and predicate logic, set theory, relations, functions, combinatorics, graph theory.

After all chapters, create four reference appendices (400--600 words each): (1)~Symbol table listing every mathematical symbol by chapter with exact LaTeX commands. (2)~Theorem index with exact names, numbers, and statements. (3)~Cross-reference guide citing exact theorem numbers across chapters. (4)~Proof techniques summary citing where each technique is demonstrated.

All chapters and appendices must use consistent notation: same empty set symbol, same function composition notation, same theorem numbering format (Theorem X.Y), same proof block structure. Appendices aggregate from all chapters---any deviation in symbol usage or naming breaks cross-referencing.
}
\end{tcolorbox}

\subsubsection*{CloudDocs}
\begin{tcolorbox}[colback=gray!5, colframe=gray!60, boxrule=0.5pt, breakable]
{\small
Write complete API documentation for a fictional cloud storage service ``CloudVault''. Module-level references and cross-cutting guides.

\textbf{Module API references:} (1)~\texttt{auth\_api.md}: Authentication---OAuth2 flows, JWT tokens, scopes, rate limiting. 15+ endpoints, 2500--3500 words. (2)~\texttt{storage\_api.md}: File storage---upload (multipart, resumable, chunked), download, metadata CRUD, versioning, sharing. 20+ endpoints, 3000--4000 words. (3)~\texttt{search\_api.md}: Search---full-text, metadata filters, cursor pagination, saved searches. 10+ endpoints, 2000--3000 words. (4)~\texttt{admin\_api.md}: Administration---user CRUD, roles, quotas, analytics, audit logs, billing, webhooks. 15+ endpoints, 2500--3500 words. (5)~\texttt{webhook\_api.md}: Events---event types, registration, payload schemas, retry policy, signature verification. 10+ endpoints, 2000--3000 words.

\textbf{Cross-cutting guides} (reference ALL modules---every endpoint, error code, schema name must match): (6)~\texttt{quickstart.md}: Getting started walkthrough. (7)~\texttt{error\_reference.md}: All error codes aggregated from all modules. (8)~\texttt{migration\_guide.md}: v1 to v2 breaking changes.

Consistent formatting: \texttt{METHOD /api/v2/path}, PascalCase types, camelCase fields, \texttt{RESOURCE\_ACTION\_ERROR} codes.
}
\end{tcolorbox}

\section{Additional experiment results}
\label{app:more_results}

Table~\ref{tab:walltime} reports wall time (seconds from first model invocation to final deliverable) and Table~\ref{tab:words} reports deliverable output in words. Deliverable output counts only final user-facing files (source code, content documents, HTML pages), excluding intermediate artifacts such as planning documents, log files, build outputs, and configuration files. The throughput metric reported in the main text (Table~\ref{tab:throughput}) is computed as the ratio of deliverable output to wall time.

\begin{table*}[h]
  \caption{Wall time (seconds) across 9 evaluation tasks. Lower is better.}
  \label{tab:walltime}
  \centering
  \small
  \setlength{\tabcolsep}{3pt}
  \begin{tabular}{l@{\hskip 4pt}cccccc@{\hskip 4pt}ccc@{\hskip 4pt}c}
    \toprule
    & \multicolumn{6}{c}{Heavy} & \multicolumn{3}{c}{Medium} & \\
    \cmidrule(lr){2-7} \cmidrule(lr){8-10}
    Method & Pixel. & Shop. & Comp. & Arcade. & Slide. & Climate. & LinAlg. & Math. & Cloud. & Overall \\
    \midrule
    SquidAgent & 580 & 404 & 344 & 439 & 1022 & 464 & 147 & 566 & 383 & 483$\pm$239 \\
    \midrule
    Claude Code & 1999 & 1000 & 1790 & 1105 & 2297 & 1000 & 529 & 757 & 614 & 1232$\pm$638 \\
    SeqCV & 1115 & 2231 & 2780 & 2620 & 3127 & 2214 & 863 & 2243 & 3490 & 2298$\pm$860 \\
    MetaGPT & 3356 & 2396 & 3696 & 4412 & 3301 & 1843 & 479 & 969 & 897 & 2372$\pm$1403 \\
    AFlow & 1353 & 534 & 949 & 395 & 168 & 189 & 307 & 401 & 307 & 511$\pm$392 \\
    Flow & 1078 & 851 & 1769 & 351 & 1178 & 766 & 270 & 1153 & 658 & 897$\pm$462 \\
    MacNet & 315 & 326 & 298 & 307 & 318 & 207 & 298 & 359 & 386 & 313$\pm$49 \\
    AgentConductor & 271 & 343 & 166 & 223 & 169 & 191 & 199 & 307 & 190 & 229$\pm$64 \\
    \bottomrule
  \end{tabular}
\end{table*}

\begin{table*}[h]
  \caption{Deliverable output (words) across 9 evaluation tasks. We count only final deliverable files (source code, content documents), excluding intermediate artifacts such as logs, planning documents, and build outputs.}
  \label{tab:words}
  \centering
  \small
  \resizebox{\textwidth}{!}{
  \setlength{\tabcolsep}{3pt}
  \begin{tabular}{l@{\hskip 4pt}rrrrrr@{\hskip 4pt}rrr@{\hskip 4pt}r}
    \toprule
    & \multicolumn{6}{c}{Heavy} & \multicolumn{3}{c}{Medium} & \\
    \cmidrule(lr){2-7} \cmidrule(lr){8-10}
    Method & Pixel. & Shop. & Comp. & Arcade. & Slide. & Climate. & LinAlg. & Math. & Cloud. & Overall \\
    \midrule
    SquidAgent & 12{,}180 & 14{,}180 & 5{,}538 & 23{,}091 & 48{,}749 & 12{,}853 & 6{,}806 & 27{,}225 & 18{,}461 & 18{,}787$\pm$13{,}260 \\
    \midrule
    Claude Code & 9{,}995 & 16{,}700 & 6{,}623 & 19{,}338 & 45{,}710 & 13{,}500 & 11{,}955 & 20{,}212 & 18{,}174 & 18{,}023$\pm$11{,}328 \\
    SeqCV & 5{,}910 & 34{,}804 & 11{,}120 & 36{,}680 & 56{,}599 & 34{,}981 & 18{,}813 & 62{,}580 & 54{,}444 & 35{,}103$\pm$20{,}269 \\
    MetaGPT & 41{,}279 & 39{,}055 & 25{,}502 & 38{,}384 & 32{,}020 & 21{,}010 & 9{,}676 & 21{,}706 & 12{,}827 & 26{,}829$\pm$11{,}570 \\
    AFlow & 5{,}681 & 4{,}696 & 3{,}511 & 5{,}844 & 2{,}133 & 1{,}796 & 5{,}157 & 7{,}332 & 4{,}728 & 4{,}542$\pm$1{,}790 \\
    Flow & 5{,}174 & 11{,}233 & 6{,}192 & 4{,}528 & 12{,}722 & 8{,}349 & 5{,}211 & 25{,}251 & 15{,}529 & 10{,}465$\pm$6{,}741 \\
    MacNet & 3{,}591 & 4{,}010 & 2{,}712 & 4{,}881 & 4{,}293 & 2{,}919 & 5{,}483 & 8{,}724 & 9{,}303 & 5{,}102$\pm$2{,}387 \\
    AgentConductor & 3{,}550 & 5{,}145 & 1{,}809 & 4{,}170 & 3{,}566 & 3{,}190 & 4{,}716 & 8{,}381 & 5{,}035 & 4{,}396$\pm$1{,}822 \\
    \bottomrule
  \end{tabular}
  }
\end{table*}

\section{Scheduler analysis}
\label{sec:analysis}

Four of the nine evaluation tasks (MathRef, SlideKit, CloudDocs, and ClimateAnalysis) contain multi-layer DAGs where the scheduler makes different parallel/serial decisions per layer. Table~\ref{tab:scheduler_decisions} reports the layer-wise alignment cost, token-cost ratio, and scheduler decision. The scheduler applies Eq.~\ref{eq:margin} independently to each layer and parallelizes only when the ratio exceeds $\alpha=1.9$.

\begin{table}[h]
  \caption{Scheduler decisions on multi-layer DAG evaluation tasks. $\widehat{C}_{\mathrm{alg}}(\mathcal{L}_k)$ is estimated separately for each layer. The scheduler parallelizes a layer only when the ratio exceeds $\alpha=1.9$.}
  \label{tab:scheduler_decisions}
  \centering
  \small
  \begin{tabular}{llccrl}
    \toprule
    task & Layer & Tasks & $\widehat{C}_{\mathrm{alg}}(\mathcal{L}_k)$ & Ratio & Decision \\
    \midrule
    \multirow{2}{*}{MathRef}
      & L0: chapters & 6 & 2500 & 5.09 & Parallel \\
      & L1: appendices & 4 & 6000 & 1.45 & \textbf{Serial} \\
    \midrule
    \multirow{2}{*}{SlideKit}
      & L0: decks & 6 & 4000 & 4.80 & Parallel \\
      & L1: handouts & 3 & 3500 & 2.00 & Parallel \\
    \midrule
    \multirow{2}{*}{CloudDocs}
      & L0: modules & 5 & 3500 & 3.31 & Parallel \\
      & L1: guides & 3 & 3000 & 1.64 & \textbf{Serial} \\
    \midrule
    \multirow{3}{*}{ClimateAnalysis}
      & L0: scripts & 4 & 2000 & 3.53 & Parallel \\
      & L1: correlation & 1 & 0 & 1.00 & \textbf{Serial} \\
      & L2: reports & 3 & 5000 & 0.93 & \textbf{Serial} \\
    \bottomrule
  \end{tabular}
\end{table}

On MathRef, the scheduler parallelizes the six heavy chapters ($\widehat{C}_{\mathrm{alg}}(\mathcal{L}_0) = 2500$, $\rho_0 = 5.09$) but serializes the four appendices ($\widehat{C}_{\mathrm{alg}}(\mathcal{L}_1) = 6000$, $\rho_1 = 1.45$). The appendices must cross-reference exact theorem numbers, notation, and definitions across all chapters, and the high alignment cost reflects this tight coupling.

On CloudDocs, the scheduler parallelizes the five independent API module references ($\widehat{C}_{\mathrm{alg}}(\mathcal{L}_0) = 3500$, $\rho_0 = 3.31$) but serializes the three cross-cutting guides ($\widehat{C}_{\mathrm{alg}}(\mathcal{L}_1) = 3000$, $\rho_1 = 1.64$), which must reference exact endpoint paths and error codes from all modules.

On ClimateAnalysis, the scheduler makes three distinct decisions: it parallelizes the independent analysis scripts ($\widehat{C}_{\mathrm{alg}}(\mathcal{L}_0) = 2000$, $\rho_0 = 3.53$), serializes the single correlation script that depends on all previous outputs, and serializes the final reports ($\widehat{C}_{\mathrm{alg}}(\mathcal{L}_2) = 5000$, $\rho_2 = 0.93$) because each report must cite exact numerical values from the analysis results.

\section{Full prompts}
\label{sec:prompts}

We provide the complete orchestrator and worker prompts used in SquidAgent (fork + bypass mode). These are the actual system prompts given to the LLM agents.

\subsection{Scheduler invocation}

The scheduler is implemented as a deterministic Python function that is automatically invoked \emph{inside} the \texttt{create\_plan\_branch} tool. When the orchestrator calls this tool with a task list, the tool: (1)~parses the task DAG from the \texttt{tasks} and \texttt{layer\_convention\_tokens} fields, (2)~runs the scheduler (computing topological layers, applying the token-cost inequality per layer with layer-specific $\widehat{C}_{\mathrm{alg}}(\mathcal{L}_k)$), and (3)~returns the execution schedule as part of the tool response. The orchestrator never calls the scheduler directly --- it simply calls \texttt{create\_plan\_branch} and receives the schedule in the response JSON. This design ensures the scheduling decision is always made deterministically and cannot be overridden or skipped by the LLM.

\subsection{Orchestrator prompt}

\begin{tcolorbox}[colback=gray!5, colframe=gray!60, boxrule=0.5pt, breakable]
{\small\ttfamily
\textbf{Orchestrator Instructions}\\[4pt]
You are an autonomous development orchestrator. All work happens on the \texttt{flow} branch.\\[4pt]
\textbf{MANDATORY Execution Mode: Parallel Worker Dispatch}\\[2pt]
You MUST spawn workers for parallel tasks. Do NOT write parallel tasks yourself with Write --- you MUST delegate them to workers. Only SERIAL tasks (as marked by the scheduler) should be written directly.\\[2pt]
After \texttt{create\_plan\_branch} returns, inspect the \texttt{schedule} field:\\
- ``parallel'' → MUST spawn workers (one per task, all in the SAME response so they run in parallel)\\
- ``serial'' → write these tasks DIRECTLY YOURSELF using Write/Edit.\\[2pt]
If a serial task depends on earlier-layer outputs, use Read
to load those dependency files in full before writing the output.
Extracting headings or matching patterns with grep is insufficient
for tasks that aggregate or cross-reference upstream content.\\[2pt]
Execution pattern:\\
1. Call \texttt{create\_plan\_branch} to get the schedule\\
2. For PARALLEL tasks: launch workers in a single response (they run concurrently)\\
3. For SERIAL tasks: write them yourself while workers run (or after, depending on dependencies)\\
4. After all tasks complete: review for cross-file consistency\\[2pt]
CRITICAL rules:\\
- Each worker writes ONLY its assigned file(s) --- no overlap\\
- Include full conventions in each worker prompt (workers share your context but be explicit)\\
- SERIAL tasks: write directly yourself, do NOT spawn workers\\
- After all workers complete: review for cross-file consistency and fix any issues\\
- Do NOT spawn a separate worker for review/verification --- do it yourself directly\\[2pt]
For multi-layer DAGs:\\
1. Layer 0 tasks first (spawn workers or write serial)\\
2. Wait for all Layer 0 workers to complete\\
3. Layer 1 tasks next (can reference Layer 0 outputs)\\
4. Repeat until all layers done\\[8pt]
\textbf{Your Autonomy}\\[2pt]
You have complete freedom to choose:\\
- Execute tasks directly yourself OR decompose and delegate to workers\\
- Create formal plans OR work ad-hoc\\
- Work sequentially OR parallelize\\
- Revise your approach mid-execution\\[4pt]
Trust your judgment. Choose the simplest effective approach.\\[8pt]
\textbf{Working with Plans}\\[2pt]
Plans are a tool, not a requirement. Use them when helpful.\\[2pt]
Creating a Plan:
\begin{lstlisting}[style=promptstyle]
create_plan_branch({
    "session_id": "session-20250117-120000",
    "user_request": "Add user authentication",
    "architecture": "MVC with Flask, SQLAlchemy, bcrypt...",
    "design_doc": "Repository pattern for data access...",
    "technology_stack": "Python 3.10, Flask 2.3, ...",
    "tasks": [
        {
            "id": "001",
            "description": "Create User model",
            "depends_on": [],
            "preconditions": [],
            "provides": ["User model class", "User.email field"],
            "files": ["models/user.py"],
            "estimated_time": "10 minutes",
            "estimated_output_tokens": 8000,
            "priority": "high"
        },
        ...
    ],
    "estimated_total_time": "45 minutes",
    "layer_convention_tokens": {"L0": 3000, "L1": 5000},
    "dependency_graph": "Ready immediately: 001, 002\n
                         After 001: 003 available\n..."
})
\end{lstlisting}
\vspace{-4pt}
Per-task \texttt{estimated\_output\_tokens} is REQUIRED. The scheduler uses it to decide parallel vs serial: if the total serial cost is much larger than the parallel cost (plan overhead + longest task), tasks run in parallel; otherwise serial. Without this field, the scheduler cannot make good decisions. Estimate how many tokens the worker LLM will output for each task (e.g., a 2000-word document $\approx$ 8000 tokens, a small config file $\approx$ 500 tokens).\\[4pt]
\texttt{layer\_convention\_tokens} is REQUIRED. For each dependency layer (L0, L1, \ldots), estimate how many output tokens the orchestrator needs to write shared conventions before launching that layer's parallel workers. Different layers may have different alignment costs --- e.g., independent chapters (L0) need fewer conventions than cross-referencing appendices (L1). This directly affects the scheduler's per-layer parallel/serial decision --- higher values make serial more likely. Guidelines per layer:\\
- Independent tasks with few shared interfaces: \textasciitilde500--2000 tokens.\\
- Moderate tasks with shared naming, formatting, or API conventions: \textasciitilde3000--4000 tokens.\\
- Tightly coupled tasks with extensive cross-file references: \textasciitilde6000--9000 tokens.\\
- Aggregation layers, such as appendices, indexes, or consolidated reports over prior-layer outputs: \textasciitilde2300--2500 tokens for a few small documents; use \textasciitilde2400 and never below 2300. Scale up for more or longer documents.\\[8pt]
\textbf{Parallelizability Scheduling (Automatic)}\\[2pt]
\texttt{create\_plan\_branch} automatically runs the scheduler and returns the schedule in its response. The scheduler uses \texttt{estimated\_output\_tokens} to decide: small tasks (\textasciitilde500--1000 tokens) run serial, large tasks (\textasciitilde5000+) run parallel.\\[2pt]
CRITICAL --- After \texttt{create\_plan\_branch} returns, IMMEDIATELY inspect the \texttt{schedule} field and follow the execution mode instructions above.\\[8pt]
\textbf{Conventions (Session Fork Mode)}\\[2pt]
For ``parallel'' layers: Session Fork is enabled --- workers inherit your full conversation context via fork, so do NOT write a CONVENTIONS.md file. Instead, you MUST output a detailed Conventions section in your response BEFORE launching workers.\\[2pt]
For a layer that depends on earlier outputs, first use Read
to load the full contents of its dependency files before writing
conventions and launching workers. Skip files already read in
this session. This makes the upstream content available in the
context inherited by workers. Grep matches alone are insufficient.
The first layer requires no such reading.\\[2pt]
This section must contain:\\
1. Concrete rules --- one unambiguous choice per convention (notation, style, structure, terminology, formatting)\\
2. Exhaustive coverage --- cover ALL ambiguous decisions that independent workers might resolve differently\\
3. Examples --- show the exact syntax for each convention\\[2pt]
Workers will see these conventions through the forked conversation history --- no file needed. After outputting conventions, proceed to launch workers immediately.\\[8pt]
\textbf{Minimize text output.} Every output token costs time and money. Do NOT output status tables, progress reports, or markdown summaries after each worker completion. A single short sentence is enough. Save output tokens for actual content.
}
\end{tcolorbox}

\subsection{Worker prompt}

\begin{tcolorbox}[colback=gray!5, colframe=gray!60, boxrule=0.5pt, breakable]
{\small\ttfamily
You are an autonomous development agent specialized in executing individual programming tasks within an isolated git worktree environment. You operate within a sophisticated workflow architecture where tasks are pre-planned, metadata-driven, and executed with rigorous design-first principles.\\[4pt]
\textbf{Your Core Identity}\\[2pt]
You are a disciplined, methodical developer who:\\
- Works in isolated git worktrees to enable parallel task execution\\
- Follows a commit-only architecture where git commits are the single source of truth\\
- Implements a design-first approach with progressive, well-documented commits\\
- Tests if needed\\
- Autonomously merges completed work\\[4pt]
\textbf{Worktree Isolation Rules}\\[2pt]
You are working in an isolated git worktree, NOT the main repository directory.\\[2pt]
ABSOLUTE REQUIREMENTS:\\
1. Begin by using \texttt{cd <worktree-path>} to enter your isolated worktree\\
2. Your task branch is ALREADY checked out\\
3. All file operations happen within your worktree directory\\
4. Read from flow branch using \texttt{git show flow:<filepath>} when needed\\[4pt]
\textbf{Your Mission Structure}\\[2pt]
You execute ONE task autonomously through this workflow:\\[2pt]
1. Navigate to worktree (cd to provided path)\\
2. Read task metadata (using \texttt{mcp\_\_git\_\_parse\_task} MCP tool)\\
3. Understand context (from plan branch using \texttt{mcp\_\_git\_\_parse\_plan})\\
4. Check available code (from flow branch using \texttt{mcp\_\_git\_\_get\_provides})\\
5. Create initial design commit (design + TODO in commit message, NO files)\\
6. Implement incrementally (commit after EACH TODO with design + TODO + progress)\\
7. Run tests (if needed)\\
8. Signal completion (TASK\_COMPLETE commit on task branch --- do NOT merge to flow)\\[4pt]
\textbf{Initial Design Commit (MANDATORY)}\\[2pt]
NON-NEGOTIABLE REQUIREMENT: You MUST create an initial design commit BEFORE any implementation. Git commits are the single source of truth.\\[2pt]
This commit must contain:\\
- \#\# Design: Architecture decisions and interfaces\\
- \#\# TODO List: Implementation checklist with [ ] markers\\
- \#\# Progress: Status tracking\\[4pt]
\textbf{Progressive Implementation}\\[2pt]
CRITICAL RULE: Commit after EACH TODO item with the FULL design and updated TODO list. Each commit includes:\\
1. \#\# Implementation: What was done in THIS commit\\
2. \#\# Design: Complete design preserved from initial commit\\
3. \#\# TODO List: Updated with [x] marking completed items\\
4. \#\# Progress: Current completion ratio (X/Y tasks)\\[4pt]
\textbf{Signal Completion (MANDATORY)}\\[2pt]
After all implementation is done, create a final completion commit ON YOUR TASK BRANCH:\\[2pt]
\texttt{git commit --allow-empty -m "TASK\_COMPLETE: task-001}\\
\texttt{\#\# Provides}\\
\texttt{- capability-1}\\
\texttt{- capability-2"}\\[2pt]
CRITICAL: Do NOT merge to flow. Do NOT checkout flow. Stay on your task branch. The orchestrator merges all task branches sequentially to avoid conflicts.\\[4pt]
\textbf{Starting Your Work}\\[2pt]
When activated, immediately:\\
1. Identify worktree path from prompt\\
2. Use \texttt{cd} to enter worktree\\
3. Check for CONVENTIONS.md --- if it exists, read it FIRST and follow ALL conventions exactly\\
4. Use \texttt{mcp\_\_git\_\_parse\_task} to read your task metadata\\
5. Begin execution protocol\\[2pt]
Your workflow ends with TASK\_COMPLETE commit. Do not attempt to merge or switch branches.
}
\end{tcolorbox}

\begingroup
\section{Robustness to ratio-estimation error}
\label{app:robustness}

The following result characterizes when bounded error in the serial-to-parallel cost ratio can change the threshold decision. It separates error-induced decision changes from the intended conservatism of the safety margin.

\begin{proposition}[Robustness to ratio-estimation error]
\label{prop:ratio_robustness}
Let $\rho=T_{\mathrm{ser}}/T_{\mathrm{par}}$ be the true serial-to-parallel cost ratio, with $T_{\mathrm{par}}>0$, and let $\widehat{\rho}$ be its estimate. Suppose
\begin{equation}
  |\widehat{\rho}-\rho|\le\varepsilon,\qquad \varepsilon\ge0.
  \label{eq:ratio_error_bound}
\end{equation}
The scheduler selects parallel execution if and only if $\widehat{\rho}>\alpha$, where $\alpha>1$. If $\rho>\alpha+\varepsilon$, the scheduler is guaranteed to select parallel execution. If $\rho<\alpha-\varepsilon$, it is guaranteed to select serial execution. Consequently, estimation error can change the decision relative to applying the same threshold to the true ratio only in the region
\begin{equation}
  |\rho-\alpha|\le\varepsilon.
  \label{eq:ratio_sensitive_region}
\end{equation}
Furthermore, if $\varepsilon<\alpha-1$, every layer with $\rho\le1$ is guaranteed to be executed serially.
\end{proposition}

\begin{proof}
If $\rho>\alpha+\varepsilon$, Eq.~\ref{eq:ratio_error_bound} yields
\begin{equation}
  \widehat{\rho}\ge\rho-\varepsilon>\alpha,
\end{equation}
so the scheduler selects parallel execution. If $\rho<\alpha-\varepsilon$, then
\begin{equation}
  \widehat{\rho}\le\rho+\varepsilon<\alpha,
\end{equation}
so the scheduler selects serial execution. These inequalities imply Eq.~\ref{eq:ratio_sensitive_region}. Finally, if $\rho\le1$ and $\varepsilon<\alpha-1$, then
\begin{equation}
  \widehat{\rho}\le\rho+\varepsilon\le1+\varepsilon<\alpha,
\end{equation}
and incorrect parallelization is impossible. By the strict scheduling rule, an estimated tie $\widehat{\rho}=\alpha$ also selects serial execution.
\end{proof}

\paragraph{Interpretation of the safety margin.}
Small estimation errors do not change the decision when the true ratio is sufficiently far from $\alpha$. The margin asymmetrically protects against parallelizing a layer whose serial execution is cheaper, but may forgo a modest parallel gain. With an exact estimate, layers with $1<\rho\le\alpha$ are deliberately serialized. Under the error bound, a truly parallel-favoring layer can be serialized only if $1<\rho\le\alpha+\varepsilon$; the interval $\alpha<\rho\le\alpha+\varepsilon$ captures additional serialization that can arise from estimation error. This is a conditional decision-stability result under Eq.~\ref{eq:ratio_error_bound}.

\section{Additional prediction and scheduling results}
\label{app:prediction_details}

\subsection{Prediction correlations and rank reversals}
We extend the prediction comparison in Section~\ref{sec:prediction} with a rank-ordering analysis, using the same 59 worker subtasks from nine evaluation tasks and the median of five independent predictions for each quantity. We compare predicted and observed subtask orderings using Kendall's $\tau_b$ and the fraction of inverted pairs, excluding tied pairs from the latter. Confidence intervals are obtained using a paired bootstrap clustered by evaluation task.

\begin{table}[h]
\centering
\small
\caption{Rank-ordering agreement with the corresponding observed quantities. Inverted-pair fractions are computed among non-tied comparisons; brackets report 95\% confidence intervals.}
\label{tab:rank_reversals}
\begin{tabular}{lcc}
\toprule
Predicted quantity & Kendall's $\tau_b\,\uparrow$ & Inverted pairs $\downarrow$ \\
\midrule
Output-token count & $0.62\ [0.30,0.79]$ & $17\%\ [8\%,33\%]$ \\
Wall-clock execution time & $0.12$ & $44\%\ [33\%,55\%]$ \\
\bottomrule
\end{tabular}
\end{table}

There are $\binom{59}{2}=1{,}711$ subtask pairs before excluding ties. Token predictions preserve ordering substantially better (Table~\ref{tab:rank_reversals}). The rank-ordering comparison uses a paired bootstrap clustered by task. The inversion rate for wall-clock predictions is 27 percentage points higher. These rank reversals cannot be repaired by a constant multiplicative correction or a strictly increasing recalibration. This supports the scoped observation that direct wall-clock estimates mis-rank execution durations in the evaluated workflows. Figure~\ref{fig:scheduler-analysis}(a) provides complementary scheduler-level evidence by comparing estimated and realized layer cost ratios.

\subsection{Safety-margin stability across nine tasks}
\label{app:margin_stability}
We repeat the planning stage three times for each of the nine tasks and record $\widehat{\rho}_k$ for each dependency layer. Table~\ref{tab:margin_stability} summarizes the mean and range of the resulting estimates. These are repeated-planning results, distinct from the original executed-plan examples in Appendix~\ref{sec:analysis}. Because the scheduler parallelizes a layer exactly when $\widehat{\rho}_k>\alpha$, a decision changes only when $\alpha$ crosses the corresponding estimated ratio.

\begin{table}[h]
\centering
\small
\setlength{\tabcolsep}{3pt}
\caption{Repeated-planning sensitivity results. Mean and range are across three planning runs; decisions use $\alpha=1.9$.}
\label{tab:margin_stability}
\begin{tabular}{lllcl}
\toprule
Task & Layer & Mean $\widehat{\rho}_k$ & Range & Decision \\
\midrule
MathRef & Chapters & 4.64 & $[4.53,4.73]$ & Parallel \\
MathRef & Appendices & 1.71 & $[1.70,1.74]$ & Serial \\
SlideKit & Decks & 4.03 & $[3.75,4.17]$ & Parallel \\
SlideKit & Handouts & 1.71 & $[1.53,1.88]$ & Serial \\
CloudDocs & Modules & 3.16 & $[3.06,3.31]$ & Parallel \\
CloudDocs & Guides & 1.63 & $[1.44,1.75]$ & Serial \\
ClimateAnalysis & Scripts & 2.96 & $[2.61,3.36]$ & Parallel \\
ClimateAnalysis & Reports & 1.13 & $[1.00,1.25]$ & Serial \\
Remaining tasks & Independent layer & --- & $[2.60,5.71]$ & Parallel \\
\bottomrule
\end{tabular}
\end{table}

The remaining tasks are LinAlgBook, PixelCraft, ArcadeBox, CompressKit, and ShopFlow. Across the 42 layer instances, tightly coupled layers have lower estimated ratios than loosely coupled layers. The resulting separation gives a stable interval of safety margins that produce the same decisions, including the default $\alpha=1.9$, which lies near the lower end of the interval. This experiment characterizes decision stability under repeated planning; the fixed-plan MathRef experiment in Section~\ref{sec:ablation} separately measures throughput under different margins.

\section{Transfer to external Flow tasks}
\label{app:flow_transfer}
To assess transfer beyond our designed benchmarks, we evaluate three practical tasks from Flow~\citep{niu2025flow}: conference website design, game development, and LaTeX Beamer generation. We compare SquidAgent with Claude Code. Quality is evaluated using task-specific automated rubrics covering executability, required components, and task-specific correctness.

\begin{table}[h]
\centering
\small
\setlength{\tabcolsep}{4pt}
\caption{Preliminary transfer evaluation on three external tasks from Flow. Wall time is in seconds, throughput in words/s, and quality in percent.}
\label{tab:flow_transfer}
\begin{tabular}{llrrrr}
\toprule
Metric & Method & Website & Game & Beamer & Mean \\
\midrule
Wall time $\downarrow$ & Claude Code & 369 & 698 & 92 & 386.3 \\
 & SquidAgent & \textbf{259} & \textbf{439} & \textbf{83} & \textbf{260.3} \\
\midrule
Throughput $\uparrow$ & Claude Code & 19.8 & 4.0 & 10.9 & 11.6 \\
 & SquidAgent & \textbf{28.1} & \textbf{6.2} & \textbf{12.0} & \textbf{15.4} \\
\midrule
Quality $\uparrow$ & Claude Code & 100.0 & 100.0 & 97.0 & 99.0 \\
 & SquidAgent & 100.0 & 100.0 & 100.0 & 100.0 \\
\bottomrule
\end{tabular}
\end{table}

These tasks are relatively lightweight and offer limited opportunities for parallelism. Nevertheless, SquidAgent reduces mean wall time from 386.3 to 260.3 seconds, corresponding to a 1.48$\times$ speedup, while increasing mean throughput from 11.6 to 15.4 words/s and maintaining comparable task-completion quality. These results provide preliminary evidence of transfer beyond the author-designed suite.
\par
\endgroup

\section{Limitations and future work}
\label{app:limitations}

\paragraph{Noisy token-budget estimation.}
The scheduler relies on the orchestrator LLM to estimate per-task output tokens and per-layer alignment costs. As shown in Figure~\ref{fig:scheduler-analysis}(a), these estimates are noisy---individual task estimates can differ from actual output by 2--5$\times$. The margin $\alpha$ partially compensates, but it is a global constant rather than being adaptive to the prediction confidence of individual layers. Future work could explore confidence-aware margins or calibration of token estimates.

{
\paragraph{External latency beyond token generation.}
Predicted output tokens are a practical proxy mainly for LLM-generation-dominated workflows and do not directly capture latency from tool execution, retrieval, environment interaction, or external API calls. These costs can vary substantially with network conditions, server load, caching, and retries, making them difficult to estimate reliably before execution. Although context inheritance can avoid some duplicated calls by allowing workers to reuse tool and retrieval results obtained during planning, it cannot eliminate task-specific external operations. Future work could incorporate environment-specific latency models or online runtime measurements into the scheduling criterion.
}

\paragraph{Layer-level granularity.}
The cost inequality operates at the DAG layer level: all tasks within a layer receive the same parallel/serial decision. This may miss opportunities when a layer contains a mix of large and small tasks that would benefit from different treatment. The current implementation partially addresses this via a threshold that forces small tasks ($<$3000 estimated tokens) into serial mode, but a finer-grained per-task decision mechanism could improve efficiency.

\paragraph{Broader impact.}
This work proposes a scheduling strategy for coordinating existing LLM agents. It does not introduce new model capabilities, training procedures, or datasets. We do not foresee direct negative societal impacts beyond those already inherent in the underlying LLMs.

\begingroup
\paragraph{Future work.}
{Future work could extend SquidAgent along three directions. First, adaptive scheduling and dependency modeling could improve decisions as worker capabilities and task relationships change. Cost estimates could be recalibrated during model adaptation~\citep{huang2025adaptive,cai2026vi,li2026teacher,wang2026estimating}, while causal discovery and effect identifiability~\citep{li2025recovery,li2025efficient,licausal}, together with causal reasoning and failure tracing~\citep{lin2026models,lin2026medcausalx,zhu2026craftcausalresponsibilityfailure}, could inform investigations of hidden dependencies and error propagation across workers. This would require additional validation, since the current DAG represents execution dependencies rather than identified causal effects. Second, trustworthy execution and information management could be integrated into the scheduling objective. Relevant mechanisms include quality checks for unreliable feedback~\citep{cai2025reinforcement,chen2026false,Huang2026RealTimeAR} and misleading visual inputs~\citep{chen2026mvibench}, safeguards against unsafe execution~\citep{wu2026internal} and privacy leakage or unauthorized content reuse~\citep{xiang2026agents,huang2026delving}, and controls for knowledge retention~\citep{huang2026gradient} and long-term memory attribution~\citep{zhang2026memmark}. Their computational overhead should be included when deciding whether parallelization is worthwhile. Third, cross-domain evaluation could test whether these scheduling principles generalize to interaction and content generation, including dialogue and selling~\citep{gao2025can,gao2024no,su2026sellmore} and image and game generation~\citep{chen2025r2ibench,huang2026guiagents}. Further evaluation could cover clinical and driving workflows~\citep{zhu2026path,hu2026research,zhu2025pathology,zhu2026medsynapse,huang2026deferred}, as well as security analytics and delay-sensitive financial prediction~\citep{wang2026trident,ma2026musk,you2026drdgrl}. Such evaluations should examine latency, quality, and robustness jointly rather than assume that gains on the current tasks transfer directly.}
\par
\endgroup

\section{Quality scoring rubrics}
\label{app:rubrics}

Each task has a task-specific rubric organized into three tiers:

\begin{itemize}[nosep]
  \item \textbf{Tier~1 --- Task Compliance (Hard Gate):} Each item maps directly to a requirement in the benchmark prompt. If any Tier~1 item is NO, the output is marked ``not acceptable.'' Items include required files, features, algorithms, and a deliverability check (correct file structure, runs/renders without errors).
  \item \textbf{Tier~2 --- Domain-Specific Correctness:} Deeper checks on implementation quality, algorithmic correctness, cross-reference accuracy, and notation consistency.
  \item \textbf{Tier~3 --- Quality and Polish:} Nice-to-have quality indicators such as responsive design, color schemes, benchmark modes, and documentation completeness.
\end{itemize}

Each item is scored binary YES (fully met with concrete evidence) or NO (anything less). Claude Opus 4.6 reads the actual output files and provides specific evidence for each judgment. Total items per benchmark range from 21 to 37.

\subsubsection*{PixelCraft}

\begin{small}
\begin{tabular}{@{}lp{11.5cm}@{}}
\toprule
\multicolumn{2}{l}{\textbf{Tier 1: Task Compliance (Hard Gate)}} \\
\midrule
1.1 & Procedurally generated world of at least 256x256 blocks \\
1.2 & Multiple biomes with distinct terrain: forest, desert, snow, plains \\
1.3 & At least 15 block types present in code \\
1.4 & Player left/right movement with jump and gravity physics \\
1.5 & Collision detection against solid blocks \\
1.6 & Health system with damage, healing, death, and respawn \\
1.7 & Hunger system that depletes over time \\
1.8 & 36-slot inventory: 9 hotbar + 27 storage with item stacking \\
1.9 & 3x3 crafting grid with at least 15 recipes \\
1.10 & At least 3 enemy types: melee chaser, ranged attacker, passive animal \\
1.11 & Day-night cycle over 10 real-time minutes \\
1.12 & Save/load world state to JSON with compression \\
1.13 & Deliverability: Python 3.10+ with Pygame only, no external assets, 1024x768 window, runs via single entry point script \\
\midrule
\multicolumn{2}{l}{\textbf{Tier 2: Game Mechanics Correctness}} \\
\midrule
2.1 & Underground caves with ore deposits at varying depths \\
2.2 & Natural structures generated: trees, water lakes, flower patches \\
2.3 & Each block has visual appearance, mining hardness, and drop items \\
2.4 & Block mining with progress animation and block placement from inventory \\
2.5 & Hotbar selection with number keys \\
2.6 & Drag-and-drop between inventory slots \\
2.7 & Enemy AI: chase, attack, wander, flee behaviors \\
2.8 & Enemies spawn at night or in dark areas \\
2.9 & Knockback and attack cooldowns on enemies \\
2.10 & Smooth sky color transitions during day-night cycle \\
2.11 & Darkness overlay at night with reduced visibility \\
2.12 & Torch placement provides local lighting \\
\midrule
\multicolumn{2}{l}{\textbf{Tier 3: Quality / Polish}} \\
\midrule
3.1 & UI: health and hunger bars displayed \\
3.2 & UI: minimap showing nearby terrain \\
3.3 & UI: debug info with FPS, coordinates, current biome \\
3.4 & Inventory screen overlay toggled with E key \\
3.5 & Multiple save slots \\
3.6 & Auto-save every 2 minutes \\
3.7 & Block breaking particles \\
3.8 & Damage numbers displayed on hit \\
3.9 & Item pickup sparkle effects \\
3.10 & Torch flame visual effects \\
3.11 & Targets 60 FPS \\
3.12 & All graphics use colored rectangles and pixels, no external assets \\
\bottomrule
\end{tabular}
\end{small}
\vspace{4pt}

\subsubsection*{ShopFlow}

\begin{small}
\begin{tabular}{@{}lp{11.5cm}@{}}
\toprule
\multicolumn{2}{l}{\textbf{Tier 1: Task Compliance (Hard Gate)}} \\
\midrule
1.1 & Product catalog with browsing, pagination, and product detail pages \\
1.2 & Filter by category, price range, and rating; search by keyword; sort by price/rating/newest \\
1.3 & Shopping cart: add with quantity, update quantities, remove, subtotals and total \\
1.4 & User registration with email/username/password, login/logout with session management \\
1.5 & Checkout with shipping address form validation and order confirmation with order number \\
1.6 & Order status tracking with all statuses: pending, paid, shipped, delivered, cancelled \\
1.7 & Admin dashboard with overview statistics and order/product management \\
1.8 & Backend: Flask + SQLAlchemy + SQLite; frontend: vanilla HTML/CSS/JS only \\
1.9 & All data loaded dynamically from REST API via fetch(), no hardcoded data \\
1.10 & Session-based authentication with server-side sessions \\
1.11 & Database seed script with 20+ products across 5 categories, plus admin and regular user \\
1.12 & Deliverability: runnable with single command `python app.py`, no external image assets, no CSS frameworks \\
\midrule
\multicolumn{2}{l}{\textbf{Tier 2: Implementation Correctness}} \\
\midrule
2.1 & Cart persists across page navigation \\
2.2 & Stock availability checked before adding to cart \\
2.3 & Product reviews: one per product per user \\
2.4 & Payment simulation: mark as paid after confirmation \\
2.5 & Cancel pending orders functionality \\
2.6 & Order history viewable by user \\
2.7 & Admin sales analytics: top products, daily orders \\
2.8 & Admin order filtering and status updates \\
2.9 & All API endpoints validate inputs and return appropriate error messages \\
2.10 & Security: password hashing, CSRF-safe session auth, SQL injection prevention via ORM \\
\midrule
\multicolumn{2}{l}{\textbf{Tier 3: Quality / Polish}} \\
\midrule
3.1 & Responsive design working on mobile and desktop \\
3.2 & Navigation with search bar, cart icon with item count, user menu \\
3.3 & Every button, form, and link works end-to-end with backend API \\
3.4 & Admin analytics with visual charts (CSS-based, no external charting libraries) \\
3.5 & CSS: responsive layout with mobile breakpoints, hover effects, transitions, CSS custom properties \\
3.6 & At least 600 lines of CSS \\
3.7 & CSS-generated placeholders for product images \\
3.8 & Frontend is visually polished, not raw unstyled HTML \\
\bottomrule
\end{tabular}
\end{small}
\vspace{4pt}

\subsubsection*{CompressKit}

\begin{small}
\begin{tabular}{@{}lp{11.5cm}@{}}
\toprule
\multicolumn{2}{l}{\textbf{Tier 1: Task Compliance (Hard Gate)}} \\
\midrule
1.1 & Huffman coding implemented from scratch with compress and decompress \\
1.2 & LZ77 implemented from scratch with compress and decompress \\
1.3 & LZW implemented from scratch with compress and decompress \\
1.4 & Burrows-Wheeler Transform with Move-to-Front encoding implemented \\
1.5 & Run-Length Encoding implemented from scratch \\
1.6 & Arithmetic coding implemented from scratch \\
1.7 & Deflate combining LZ77 and Huffman implemented \\
1.8 & Round-trip correctness: decompressing a compressed file produces exact original for all algorithms \\
1.9 & Custom archive format for bundling multiple files with filenames, sizes, checksums, compressed data \\
1.10 & Command-line interface with compress, decompress, archive create/extract, and benchmark mode \\
1.11 & No external compression libraries used \\
1.12 & Deliverability: Python 3.10+, runnable from command line, no external libraries \\
\midrule
\multicolumn{2}{l}{\textbf{Tier 2: Algorithm Correctness}} \\
\midrule
2.1 & All algorithms handle edge cases: empty input, single byte, all identical bytes, random bytes \\
2.2 & CRC32 checksum using lookup table implemented \\
2.3 & Adler32 rolling checksum implemented \\
2.4 & Integrity verification on decompression \\
2.5 & Archive supports compressing entire directories preserving structure \\
2.6 & Archive integrity verification on extraction \\
2.7 & Decompression auto-detects algorithm from header \\
2.8 & Bit-level I/O utility for sub-byte operations \\
\midrule
\multicolumn{2}{l}{\textbf{Tier 3: Quality / Polish}} \\
\midrule
3.1 & Benchmark mode reports compression ratio, speed, and correctness for all algorithms \\
3.2 & Built-in self-test generating various data patterns (random, repeated, text-like) and verifying round-trip \\
3.3 & Show file and archive metadata command \\
3.4 & CLI provides clear usage/help messages for all subcommands \\
\bottomrule
\end{tabular}
\end{small}
\vspace{4pt}

\subsubsection*{ArcadeBox}

\begin{small}
\begin{tabular}{@{}lp{11.5cm}@{}}
\toprule
\multicolumn{2}{l}{\textbf{Tier 1: Task Compliance (Hard Gate)}} \\
\midrule
1.1 & Tetris game present with 10x20 grid, 7 standard tetromino shapes, gravity, level-up every 10 lines, scoring 100/300/500/800 \\
1.2 & Minesweeper game present with 3 difficulty levels (9x9/10, 16x16/40, 30x16/99), first click always safe, flood-fill reveal \\
1.3 & 2048 game present with 4x4 grid, tile merging, new tile probabilities (90\% 2, 10\% 4), win at 2048 \\
1.4 & Snake game present with grid movement, growth on food, game over on self/wall collision \\
1.5 & Breakout game present with paddle, ball, 5 rows of bricks, 3 lives, ball angle varies on paddle hit \\
1.6 & Sudoku game present with puzzle generator, 3 difficulty levels, click-to-fill, conflict highlighting \\
1.7 & Memory Match game present with 3 difficulty levels (4x4, 4x6, 6x6), flip matching logic with 1s delay \\
1.8 & Gomoku game present with 15x15 board, AI opponent using minimax with alpha-beta pruning \\
1.9 & Main menu/arcade shell for game selection \\
1.10 & All 8 games use HTML5 Canvas for rendering \\
1.11 & Pure HTML/CSS/JS only, no frameworks/libraries/build tools/external assets \\
1.12 & Deliverability: works by opening index.html directly in a browser, no server required \\
\midrule
\multicolumn{2}{l}{\textbf{Tier 2: Game Implementation Correctness}} \\
\midrule
2.1 & Tetris: next piece preview, ghost piece, wall kicks on rotation \\
2.2 & Minesweeper: proper number colors (1=blue, 2=green, 3=red, etc.), timer, mine counter \\
2.3 & 2048: smooth tile sliding/merging animations using requestAnimationFrame \\
2.4 & Snake: increasing speed as snake gets longer \\
2.5 & Breakout: AABB collision detection, ball speed increases on paddle hit \\
2.6 & Sudoku: pencil marks mode toggle for noting candidates \\
2.7 & Memory Match: cards display emoji symbols, track moves and time \\
2.8 & Gomoku: heuristic evaluation scoring patterns (open-four, half-open-four, open-three, etc.) \\
2.9 & Every game has pause (Escape/P), restart button, back-to-menu button \\
2.10 & High scores stored in localStorage, displayed on main menu and in-game \\
2.11 & Global statistics tracked: total games played, total time, best scores \\
\midrule
\multicolumn{2}{l}{\textbf{Tier 3: Quality / Polish}} \\
\midrule
3.1 & All rendering at 60 FPS using requestAnimationFrame \\
3.2 & Canvas scales to fit viewport while maintaining aspect ratio \\
3.3 & Keyboard controls do not scroll the page (preventDefault on arrow keys, space) \\
3.4 & Each game file is at least 300 lines of JavaScript with complete logic \\
3.5 & All game logic, physics, collision detection, and AI implemented from scratch \\
\bottomrule
\end{tabular}
\end{small}
\vspace{4pt}

\subsubsection*{SlideKit}

\begin{small}
\begin{tabular}{@{}lp{11.5cm}@{}}
\toprule
\multicolumn{2}{l}{\textbf{Tier 1: Task Compliance (Hard Gate)}} \\
\midrule
1.1 & day1\_foundations.html exists: ML Foundations covering supervised vs unsupervised, bias-variance, cross-validation, evaluation metrics, feature engineering \\
1.2 & day1\_linear.html exists: Linear Models covering linear/logistic regression, regularization, SVMs \\
1.3 & day2\_trees.html exists: Tree-Based Methods covering decision trees, random forests, gradient boosting \\
1.4 & day2\_neural.html exists: Neural Networks covering perceptrons, backprop, CNNs, RNNs, transfer learning \\
1.5 & day3\_practical.html exists: Practical ML Pipeline covering preprocessing, model selection, deployment, ethics \\
1.6 & day3\_advanced.html exists: Advanced Topics covering attention, transformers, generative models, RL \\
1.7 & handout.html exists: printable summary referencing exact equation numbers and terminology from all 6 presentations \\
1.8 & exercises.html exists: 3 exercises per session (18 total) referencing specific slides \\
1.9 & glossary.html exists: all technical terms from all 6 presentations with matching definitions \\
1.10 & Each main presentation has 20-30 slides, inline CSS and JS, speaker notes in hidden divs \\
1.11 & Each main presentation is at least 500 lines \\
1.12 & Deliverability: all files are standalone HTML with CSS for styling and JS for slide navigation (arrow keys, click); opens directly in browser \\
\midrule
\multicolumn{2}{l}{\textbf{Tier 2: Content Correctness and Cross-Referencing}} \\
\midrule
2.1 & Presentations include LaTeX-rendered equations using Unicode math symbols \\
2.2 & Presentations include SVG diagrams as specified (decision boundaries, trees, network architectures, pipeline flows) \\
2.3 & Support materials use exact same terminology, equation notation, and diagram labeling as presentations \\
2.4 & Glossary definitions match exactly how terms were introduced in source presentations \\
2.5 & Exercises reference specific slides and concepts from the corresponding presentation \\
2.6 & Support materials cross-reference using consistent session numbering (Day X, Session Y) and slide numbers \\
\midrule
\multicolumn{2}{l}{\textbf{Tier 3: Quality / Polish}} \\
\midrule
3.1 & All files share the same CSS color scheme, font choices, and slide layout conventions \\
3.2 & Support materials (handout, exercises, glossary) are 200-300 lines each \\
3.3 & Speaker notes present in hidden divs in all 6 main presentations \\
3.4 & Slide navigation works with both arrow keys and click in all presentations \\
\bottomrule
\end{tabular}
\end{small}
\vspace{4pt}

\subsubsection*{ClimateAnalysis}

\begin{small}
\begin{tabular}{@{}lp{11.5cm}@{}}
\toprule
\multicolumn{2}{l}{\textbf{Tier 1: Task Compliance (Hard Gate)}} \\
\midrule
1.1 & temperature\_analysis.py exists: loads/generates data, computes annual means, 10-year moving averages, linear regression trend, seasonal decomposition, identifies hottest/coldest years \\
1.2 & precipitation\_analysis.py exists: 6 regions, monthly/annual stats, drought detection, inter-region correlation, Mann-Kendall trend test \\
1.3 & co2\_analysis.py exists: monthly data, annual growth rates, seasonal cycle extraction, exponential curve fitting, future projections \\
1.4 & sea\_level\_analysis.py exists: 5 stations, station trends, global mean, acceleration detection (quadratic vs linear), confidence intervals \\
1.5 & correlation\_analysis.py exists: imports data from all 4 scripts, Pearson and Spearman correlations, lag correlation, Granger causality test \\
1.6 & executive\_summary.md exists: cites exact key findings from each analysis script \\
1.7 & methodology.md exists: documents all statistical methods matching implementations \\
1.8 & data\_dictionary.md exists: all variables, column names, units, data ranges matching scripts exactly \\
1.9 & Pure Python only: no pandas, numpy, or external data libraries \\
1.10 & All scripts generate synthetic data inline (no external CSV files required) \\
1.11 & Deliverability: each script is standalone and runnable, outputs formatted tables and ASCII visualizations \\
\midrule
\multicolumn{2}{l}{\textbf{Tier 2: Analytical Correctness}} \\
\midrule
2.1 & Linear regression implemented from scratch (not imported) \\
2.2 & Mann-Kendall trend test implemented from scratch \\
2.3 & Exponential curve fitting implemented from scratch \\
2.4 & Confidence intervals computed from scratch \\
2.5 & Granger causality test implemented from scratch \\
2.6 & correlation\_analysis.py imports data generation functions from the other 4 scripts \\
2.7 & Reports cite exact numbers matching script outputs \\
2.8 & Methodology formulas match the actual implementations in scripts \\
\midrule
\multicolumn{2}{l}{\textbf{Tier 3: Quality / Polish}} \\
\midrule
3.1 & Consistent variable naming across all scripts \\
3.2 & Consistent date formats: ISO 8601 \\
3.3 & Consistent units: deg-C for temperature, mm for precipitation, ppm for CO2, mm for sea level \\
3.4 & Consistent output formatting: same table border style, same decimal places \\
3.5 & Scripts output ASCII charts/visualizations \\
3.6 & Scripts are within specified line counts (400-600 lines for analysis scripts; 300-500 for correlation) \\
3.7 & Reports are within specified word counts (500-800 words for summary and methodology; 400-600 for dictionary) \\
\bottomrule
\end{tabular}
\end{small}
\vspace{4pt}

\subsubsection*{LinAlgBook}

\begin{small}
\begin{tabular}{@{}lp{11.5cm}@{}}
\toprule
\multicolumn{2}{l}{\textbf{Tier 1: Task Compliance (Hard Gate)}} \\
\midrule
1.1 & vectors.md exists: covers vector spaces, operations, linear independence, spanning sets \\
1.2 & matrix\_ops.md exists: covers matrix multiplication, inverse, transpose, row reduction \\
1.3 & determinants.md exists: covers determinant definition, properties, cofactor expansion, Cramer's rule \\
1.4 & eigen.md exists: covers eigenvalues, eigenvectors, diagonalization, characteristic polynomial \\
1.5 & svd.md exists: covers SVD, low-rank approximation, applications \\
1.6 & pca.md exists: covers PCA, covariance matrix, dimensionality reduction \\
1.7 & Each file is 800-1200 words \\
1.8 & Each file contains at least 2 LaTeX worked examples \\
1.9 & Deliverability: exactly 6 markdown files with correct filenames, valid markdown that renders correctly \\
\midrule
\multicolumn{2}{l}{\textbf{Tier 2: Content Correctness}} \\
\midrule
2.1 & Vector notation is consistent across all 6 files: same bold-vs-arrow convention \\
2.2 & Default vector orientation convention is consistent across all files \\
2.3 & Subscript start convention is consistent (0-indexed or 1-indexed) across all files \\
2.4 & Transpose symbol is consistent across all files \\
2.5 & Norm bracket style is consistent across all files \\
2.6 & Matrix size convention is consistent (m x n vs n x m) across all files \\
2.7 & Eigenvalue ordering convention is consistent across all files \\
2.8 & LaTeX worked examples are mathematically correct (no computational errors) \\
\midrule
\multicolumn{2}{l}{\textbf{Tier 3: Quality / Polish}} \\
\midrule
3.1 & Tutorial is beginner-friendly, targeting first-year undergraduates \\
3.2 & Logical progression from vectors through PCA (earlier files do not reference later concepts) \\
3.3 & Consistent formatting and heading structure across all 6 files \\
3.4 & LaTeX renders correctly in standard markdown viewers \\
\bottomrule
\end{tabular}
\end{small}
\vspace{4pt}

\subsubsection*{MathRef}

\begin{small}
\begin{tabular}{@{}lp{11.5cm}@{}}
\toprule
\multicolumn{2}{l}{\textbf{Tier 1: Task Compliance (Hard Gate)}} \\
\midrule
1.1 & Chapter on propositional and predicate logic exists: truth tables, connectives, tautologies, normal forms, quantifiers \\
1.2 & Chapter on set theory exists: set operations, power sets, Cartesian products, set identities with proofs, Russell's paradox \\
1.3 & Chapter on relations exists: reflexive, symmetric, transitive, antisymmetric, equivalence relations, partial orders, Hasse diagrams \\
1.4 & Chapter on functions exists: injective, surjective, bijective, composition, inverse, pigeonhole principle, cardinality \\
1.5 & Chapter on combinatorics exists: permutations, combinations, binomial theorem, inclusion-exclusion, generating functions, recurrences \\
1.6 & Chapter on graph theory exists: adjacency matrices, isomorphism, Euler/Hamilton paths, planar graphs, coloring, trees, MST \\
1.7 & Each chapter is 3000-4000 words \\
1.8 & Each chapter contains at least 5 fully worked LaTeX proofs \\
1.9 & Symbol table appendix exists: lists every symbol used, organized by chapter, with exact LaTeX commands matching each chapter \\
1.10 & Theorem index appendix exists: lists every theorem/lemma/corollary/proposition with exact name, number, and statement as written in source chapter \\
1.11 & Cross-reference guide appendix exists: cites exact theorem numbers showing cross-chapter concept usage \\
1.12 & Proof techniques summary appendix exists: cites exact theorem numbers for each technique \\
1.13 & Deliverability: all outputs are separate markdown files, valid markdown, correct structure of 6 chapters + 4 appendices \\
\midrule
\multicolumn{2}{l}{\textbf{Tier 2: Cross-Reference Consistency}} \\
\midrule
2.1 & Consistent theorem numbering format (Theorem X.Y) across all chapters \\
2.2 & Same symbol for empty set used in all chapters and appendices \\
2.3 & Same notation for function composition in all chapters and appendices \\
2.4 & Same proof block structure in all chapters \\
2.5 & Symbol table LaTeX commands match source chapters character-for-character \\
2.6 & Theorem index statements match source chapters exactly \\
2.7 & Cross-reference guide theorem numbers are valid and exist in source chapters \\
2.8 & Proof techniques summary theorem numbers are valid and match the technique claimed \\
\midrule
\multicolumn{2}{l}{\textbf{Tier 3: Quality / Polish}} \\
\midrule
3.1 & All proofs are mathematically correct (no logical errors) \\
3.2 & Appendices are within specified word counts (400-600 for symbol table, theorem index, proof techniques; 300-500 for cross-reference guide) \\
3.3 & Content is appropriate for undergraduate level \\
3.4 & LaTeX renders correctly in standard markdown viewers \\
3.5 & Proof techniques appendix covers all listed techniques: direct proof, contradiction, induction, contrapositive, cases, diagonalization \\
\bottomrule
\end{tabular}
\end{small}
\vspace{4pt}

\subsubsection*{CloudDocs}

\begin{small}
\begin{tabular}{@{}lp{11.5cm}@{}}
\toprule
\multicolumn{2}{l}{\textbf{Tier 1: Task Compliance (Hard Gate)}} \\
\midrule
1.1 & auth\_api.md exists: OAuth2 flows (authorization code, client credentials, refresh token), JWT structure, scopes/permissions, API key management, rate limiting, 15+ endpoints \\
1.2 & storage\_api.md exists: upload (multipart, resumable, chunked), download, metadata CRUD, versioning, trash/restore, sharing links, 20+ endpoints \\
1.3 & search\_api.md exists: full-text search, metadata filters, pagination (cursor-based), saved searches, 10+ endpoints \\
1.4 & admin\_api.md exists: user management, org settings, usage analytics, audit logs, billing, webhooks, 15+ endpoints \\
1.5 & webhook\_api.md exists: event types (file.created, file.deleted, file.shared, user.joined, storage.quota\_exceeded), registration, payload schemas, retry policy, signature verification, 10+ endpoints \\
1.6 & quickstart.md exists: walks through register app, get API key, authenticate, upload file, search, share, set up webhook \\
1.7 & error\_reference.md exists: aggregates ALL error codes from all 5 modules \\
1.8 & migration\_guide.md exists: v1 to v2 breaking changes across all modules \\
1.9 & Each module API doc includes full request/response examples in curl and Python \\
1.10 & Deliverability: all outputs are separate markdown files, valid markdown, 5 module docs + 3 guides \\
\midrule
\multicolumn{2}{l}{\textbf{Tier 2: Cross-Reference Consistency}} \\
\midrule
2.1 & Quickstart references specific endpoints from module docs by exact path and method \\
2.2 & Error reference lists each error with HTTP status, code string, description, which endpoints return it, resolution steps \\
2.3 & Error codes in error\_reference.md match those defined in each module doc exactly \\
2.4 & Migration guide references exact endpoint paths from module docs \\
2.5 & Consistent endpoint format: METHOD /api/v2/path \\
2.6 & Consistent schema naming: PascalCase for types, camelCase for fields \\
2.7 & Consistent error code format: RESOURCE\_ACTION\_ERROR \\
2.8 & Same authentication header examples throughout all documents \\
\midrule
\multicolumn{2}{l}{\textbf{Tier 3: Quality / Polish}} \\
\midrule
3.1 & Module docs are within specified word counts (auth: 2500-3500; storage: 3000-4000; search: 2000-3000; admin: 2500-3500; webhook: 2000-3000) \\
3.2 & Guide docs are within specified word counts (quickstart: 800-1200; error ref: 800-1200; migration: 600-1000) \\
3.3 & Module docs meet minimum endpoint counts (auth: 15+; storage: 20+; search: 10+; admin: 15+; webhook: 10+) \\
3.4 & Migration guide lists renamed endpoints, changed schemas, deprecated fields, new required parameters \\
3.5 & Consistent formatting throughout all documents (heading levels, code blocks, table styles) \\
\bottomrule
\end{tabular}
\end{small}
\vspace{4pt}

\end{document}